\documentclass[journal]{IEEEtran}

\usepackage[utf8]{inputenc}
\usepackage{xcolor}
\usepackage{graphicx}   
\usepackage{amsmath,amssymb, amsfonts, mathtools}
\mathtoolsset{showonlyrefs}
\usepackage[utf8]{inputenc}
\usepackage[english]{babel}
\usepackage[caption=false,font=footnotesize]{subfig}
\usepackage{mathrsfs}
\usepackage{algorithm}
\usepackage{algpseudocode}
\usepackage{hyperref}

\usepackage{amsmath,amsthm,amssymb}
\usepackage{mathtools}
\mathtoolsset{showonlyrefs}
\usepackage{mathrsfs}
\usepackage{comment}
\usepackage{array}

\newtheorem{theorem}{Theorem}
\newtheorem{definition}{Definition}
\newtheorem{corollary}{Corollary}
\newtheorem{remark}{Remark}
\newtheorem{lemma}{Lemma}
\newtheorem{proposition}{Proposition}

\DeclareMathOperator{\co}{co}
\DeclareMathOperator*{\argmin}{arg\,min}
\newcommand{\norm}[1]{\left\lVert#1\right\rVert}

\newcommand{\todo}[1]{{#1}}
\definecolor{blue}{rgb}{0.0, 0.0, 1.0}

\newcolumntype{L}[1]{>{\raggedright\let\newline\\\arraybackslash\hspace{0pt}}m{#1}}
\newcolumntype{C}[1]{>{\centering\let\newline\\\arraybackslash\hspace{0pt}}m{#1}}
\newcolumntype{R}[1]{>{\raggedleft\let\newline\\\arraybackslash\hspace{0pt}}m{#1}}

\begin{document}

\title{Data-Driven Robust Barrier Functions for Safe, Long-Term Operation}

\author{
        Yousef Emam$^{1}$, Paul Glotfelter$^{2}$, Sean Wilson$^{1}$, Gennaro Notomista$^{1}$, Magnus Egerstedt$^{1}$
  \thanks{This research was sponsored by Award No. 1724058 from the National Science Foundation and by the Army Research Lab through ARL DCIST CRA W911NF-17-2-0181.}
  \thanks{$^{1}$Y. Emam, S. Wilson, G. Notomista and M. Egerstedt are with the Institute for Robotics and Intelligent Machines, Georgia
    Institute of Technology, Atlanta, GA 30332, USA,
    \{emamy, sean.t.wilson, g.notomista, magnus\}@gatech.edu.}%
  \thanks{$^{2}$P. Glotfelter is with Optimus Ride, MA 0$7$10, Massachusetts, USA {pglotfel@gmail.com}}
}



\maketitle

\begin{abstract}
Applications that require multi-robot systems to operate independently for extended periods of time in unknown or unstructured environments face a broad set of challenges, such as hardware degradation, changing weather patterns, or unfamiliar terrain.
To operate effectively under these changing conditions, algorithms developed for long-term autonomy applications require a stronger focus on robustness.
Consequently, this work considers the ability to satisfy the operation-critical constraints of a disturbed system in a modular fashion, which means compatibility with different system objectives and disturbance representations.
Toward this end, this paper introduces a controller-synthesis approach to constraint satisfaction for disturbed control-affine dynamical systems by utilizing Control Barrier Functions (CBFs).  The aforementioned framework is constructed by modelling the disturbance as a union of convex hulls and leveraging previous work on CBFs for differential inclusions. 
This method of disturbance modeling grants compatibility with different disturbance-estimation methods.  For example, this work demonstrates how a disturbance learned via a Gaussian process may be utilized in the proposed framework.
These estimated disturbances are incorporated into the proposed controller-synthesis framework which is then tested on a fleet of robots in different scenarios.

\end{abstract} 

\begin{IEEEkeywords}
Robust/Adaptive Control of Robotic Systems,  Robot Safety, Multi-Robot Systems, Learning and Adaptive Systems, Collision Avoidance.
\end{IEEEkeywords} 

\IEEEpeerreviewmaketitle
\section{Introduction}
\label{sec:intro}

\IEEEPARstart{T}{he} deployment of multi-robot teams in real-world applications often requires operating in dynamic environments for extended periods of time. Examples of such real-world deployment scenarios include search and rescue \cite{kitano1999robocup} and precision agriculture \cite{tokekar2016sensor}.  In these scenarios, it is often difficult to model the system exactly due to environmental disturbances, such as varying terrain, weather patterns, and intrinsic changes in the dynamics of the robots (e.g., motor degradation).
In addition to degrading the performance of the robots, these disturbances can even lead to catastrophic failures for safety-critical systems.  Indeed, if these disturbances are not addressed, the possibility of failure becomes almost assured if the robots are required to operate over long time horizons. Therefore, long-term deployment motivates the need for robust control frameworks that can efficiently account for these uncertainties in a rigorous manner. 

Modeling uncertainty as a disturbance inherently offers a factored approach: estimating the disturbance and controlling a disturbed system.  Some established methods for addressing these issues are robust- and adaptive-control techniques \cite{freeman2008robust, fukushima2007adaptive, bemporad1999robust}, designed around system models which can be seen as the combination of a nominal component and a disturbance quantifying the amount by which the true system might differ from the nominal one.  

A different body of work in the literature attempts to deal with disturbed dynamical systems by solely relying on learning methods. For example, one approach is to directly learn the inverse dynamics of the disturbed system and then use the output of the learning model to control the system as in \cite{nguyen2009model}. Another such approach, under the end-to-end learning category, is to first learn the dynamics of the system and then use model-based Reinforcement Learning (RL) to obtain the desired controller \cite{deisenroth2013gaussian}. Typically, one benefit of these learning-based approaches is that they require less domain knowledge and are able to actively improve the performance of the controller \cite{deisenroth2013gaussian}. However, despite their flexibility, many data-driven methods suffer from a major deficiency compared to their model-based counterparts; namely they do not provide safety guarantees, and as such, their applicability to safety-critical control systems is, by nature, limited \cite{vinogradska2016stability}. 

As a way to mitigate this issue, a variety of \textit{safe-learning} approaches have been proposed in the literature (e.g., \cite{berkenkamp2017safe,  khansari2014learning}). These methods, however, lack flexibility in the following sense.  For example, \cite{ravanbakhsh2017learning} introduces a framework for learning Control Lyapunov Functions (CLF) and a so-called verifier that validates them. However, when the model of the system is inaccurate, the verifier needs to check an infinite number of inequalities throughout the state space which is computationally intractable \cite{ito2017second}. In contrast, in our previous work \cite{emam2019robust}, we presented a framework that utilizes convex hulls to model disturbances in the context of a Control Barrier Function (CBF) based safety framework.  The use of convex hulls renders the problem computationally tractable, since the uncountable disturbance set can be exactly accounted for by checking the extreme points of the convex hull, which, intuitively, represent the worst-case scenarios. The supporting theoretical results of this work leverage recent work on CBFs for differential inclusions \cite{PaulNBF, glotfelter2018}.  This approach has the benefit of being more compatible with data-driven, learning approaches (e.g., \cite{berkenkamp2016safe,wang2018safe}) which we demonstrate in this paper, as well as the classical techniques for estimating the state of an uncertain system (e.g., applying a Kalman filter).  Additionally, the usage of CBFs allows constraint satisfaction in a modular, objective-independent fashion. In this paper, we extend and generalize the approach in \cite{emam2019robust} by presenting two different ways the disturbance can affect the dynamical system, and discuss how this method can be coupled with data-driven frameworks for disturbance estimation. \todo{Specifically, for control affine dynamical systems, we address the disturbances on the drift and control dynamics respectively.} 

Prior work on barrier functions has mainly addressed smooth barrier functions, formulating the associated forward-invariance results with respect to continuous dynamical systems.  Moreover, some work has focused on robust barrier functions for uncertain systems such as \cite{nguyen2016optimal, gurriet2018invariance}.  As opposed to these works, where the disturbance is assumed to take the form of an $\epsilon$-ball (bounded in some norm), we allow for a significantly wider class of disturbances and model the disturbed system through differential inclusions. 

\todo{Similarly, in \cite{taylor2020adaptive}, the authors introduce adaptive CBFs~(aCBFs) which ensure the forward invariance of a desired safe set for systems with structured parametric uncertainty through parameter adaptation. Additionally, in \cite{9129764}, the authors introduce Robust aCBFs (RaCBFs) which build on \cite{taylor2020adaptive} and remedy the overly conservative behavior caused by the restrictive condition required by aCBFs. The authors from \cite{9129764} also combine RaCBFs with Set Membership IDentification (SMID) to estimate the uncertainty using data. Although \cite{9129764} is the most similar to our work in spirit, aCBFs and RaCBFs can solely be applied to systems where the uncertainty is a structured parametric uncertainty only affecting the drift dynamics. Moreover, although less restrictive than aCBFs, RaCBFs guarantee the forward invariance of a tightened safe set and can still be conservative if the maximum parameter error is large. In contrast, our approach assumes the uncertainty can be modelled as the union of convex hulls affecting either the drift or the control dynamics.}

To summarize, the main contributions of this paper with respect to prior work are the following: (i)~we build on a novel robust CBF formulation initially introduced in
\cite{emam2019robust}. This formulation efficiently accounts for disturbances in the dynamical system, in an exact manner, by modelling the uncertainty using convex sets that can be written as the convex hull of a finite number of points. Moreover, the formulation can also account for unions of convex sets which allows the method to encompass a wide variety of disturbances. In this paper, we expand on this methodology by providing results for different types of disturbances affecting the system. (ii)~We demonstrate how the robust CBF formulation can be coupled with data-driven methods for sample efficient online disturbance estimation. (iii)~A controller-synthesis procedure is introduced using the proposed robust CBFs along with Gaussian Processes for disturbance estimation. (iv)~The results are then specialized to differential-drive robots for an experimental application.

The paper is organized as follows. Section~\ref{sec:background} introduces the necessary background material and notation for the paper. In Section~\ref{sec:barrier-functions-for-disturbed}, we model various types of disturbances pertaining to control-affine systems and derive sufficient conditions for the robustness of the forward-invariance property (via CBFs) with respect to each of the disturbed control system.  In Section~\ref{sec:disturb-estimation}, we develop a method for estimating such disturbances through the use of Gaussian Processes in a way that is conducive to the use of the novel robust CBFs.  Furthermore, Section~\ref{sec:control-synthesis-via} provides some controller-synthesis results with respect to disturbed control systems.  Lastly, in Sections \ref{sec:robust-collision-avoidance} and \ref{sec:experiments}, we formulate the controller-synthesis procedure for differential-drive robots and apply the proposed method to the Robotarium \todo{\cite{pickem2017robotarium}, a remotely accessible multi-robot testbed,} in two different experiments.  Section~\ref{sec:conclusion} concludes the paper.

\section{Background Material} 
\label{sec:background}
In this section, we introduce the notation and background material used in the remainder of the paper. Namely, we introduce the theory of Control Barrier Functions (CBFs) and differential inclusions. 

\subsection{Notation}
\label{subsec:notation}

The notation $\mathbb{R}_{\geq a}$ represents the set of nonnegative real numbers greater or equal to $a$.  The expression $B(x', \delta)$ denotes an open ball of radius $\delta$ centered on a point $x' \in \mathbb{R}^{n}$.  The operation $\co$ represents the convex hull of a set.  Given a set $A$, $2^{A}$ denotes the power set of $A$.  \todo{We define the convex hull of $p > 0$ points $\psi_i \in \mathbb{R}^n, \; \forall i~\in~\{1, \ldots, p\}$, as  
\begin{align}
\co \Psi &= \co \{\psi_1, \ldots, \psi_p\} \\
            &= \{ \sum^p_{i=0} \theta_i \psi_i | \sum^p_{i=0} \theta_i = 1, 0 \leq \theta_i \leq 1, \forall i\}.
\end{align}} 
A function $\alpha : \mathbb{R} \to \mathbb{R}$ is extended class-$\mathcal{K}$ if $\alpha$ is continuous, strictly increasing, and $\alpha(0) = 0$.  A function $\beta : \mathbb{R}_{\geq 0} \times \mathbb{R}_{\geq 0} \to \mathbb{R}_{\geq 0}$ is class-$\mathcal{KL}$ if it is class-$\mathcal{K}$ in its first argument and, for each fixed $r$, $\beta(r, \cdot)$ is continuous, decreasing, and $\lim_{s \to \infty} \beta(r, s) = 0$.

\subsection{Control Barrier Functions} 
\label{subsec:control-barrier-functions}

Control Barrier Functions (CBFs) are formulated with respect to control systems \cite{ames2014,xu2015,AmesBarriers}, and this work considers control-affine systems
\begin{equation}
    \label{eq:control-affine}
    \dot{x}(t) = f(x(t)) + g(x(t))u(x(t)) , x(0) = x_{0} ,
\end{equation}
where $f : \mathbb{R}^{n} \to \mathbb{R}^{n}$, $g : \mathbb{R}^{n} \to \mathbb{R}^{n \times m}$, and $u : \mathbb{R}^{n} \to \mathbb{R}^{m}$ are continuous.  These types of systems capture many robotic systems (e.g., differential-drive robots, quadrotors, autonomous vehicles) and remain amenable to controller synthesis. A set $\mathcal{C}$ is called forward invariant with respect to \eqref{eq:control-affine} if given any solution (potentially nonunique) to \eqref{eq:control-affine} $x : [0, t_{1}] \to \mathbb{R}^{n}$,
\begin{equation}
    x_{0} \in \mathcal{C} \implies x(t) \in \mathcal{C}, \forall t \in [0, t_{1}] .
\end{equation}

Barrier functions guarantee forward invariance of a particular set that typically represents a constraint, such as collision avoidance or connectivity maintenance.   Specifically, a barrier function is a continuously differentiable function $h : \mathbb{R}^{n} \to \mathbb{R}$ (sometimes referred to as a candidate barrier function), and the so-called safe set $\mathcal{C} \subset \mathbb{R}^{n}$ is defined as the super-zero level set to $h$
\begin{equation}
    \mathcal{C} = \{x' \in \mathbb{R} : h(x') \geq 0\} .
\end{equation}
Now, the goal becomes to ensure the forward set invariance of $\mathcal{C}$, which can be done equivalently by guaranteeing positivity of $h$ along trajectories.

Positivity can be shown if there exists a locally Lipschitz extended class-$\mathcal{K}$ function $\alpha : \mathbb{R} \to \mathbb{R}$ and a continuous function $u : \mathbb{R}^{n} \to \mathbb{R}^{m}$ such that
\begin{equation}
    \label{eq:barrier-certificate-reg}
    \nabla h(x')^{\top}(f(x') + g(x')u(x')) \geq -\alpha(h(x')), \forall x' \in \mathbb{R}^{n}.
\end{equation}
Then, $h$ is called a valid CBF for \eqref{eq:control-affine} \cite{AmesBarriers}. Moreover, \eqref{eq:barrier-certificate-reg} does not explicitly account for any uncertainty in the system.  As such, real-world disturbances (e.g., packet loss or wheel slip) can cause the system to violate the constraint.  Toward resolving this issue, the next section formulates an analogous result in the context of uncertain control systems.

\subsection{Differential Inclusions}
\label{subsec:differential-inclusions}

Differential inclusions are a generalization of differential equations that have been used to represent a variety of problems including perturbed Ordinary Differential Equations (ODEs) as in \eqref{eq:control-affine} and discontinuous dynamical systems \cite{PaulNBF}. Uncertain or disturbed systems, as addressed in this paper, fall into a particular class of differential inclusions.  As such, this section presents the high-level theory of differential inclusions, and later sections formulate the disturbed system that this work considers.

In general, differential inclusions are formulated as 
\begin{equation} 
    \label{eq:diffInc}
    \dot{x}(t) \in F(x(t)), x(0) = x_{0} ,
\end{equation}
where $F : \mathbb{R}^{n} \to 2^{\mathbb{R}^{n}}$ is an upper semi-continuous set-valued map that takes nonempty, convex, and compact values.  A set-valued map is called upper semi-continuous if for every $\epsilon > 0$, $x' \in \mathbb{R}^{n}$ there exists $\delta > 0$ such that 
\begin{equation}
    F(y) \subset F(x') + B(0, \epsilon), \forall y \in B(x', \delta) .
\end{equation}
Note that the term $F(\cdot) + B(0, \epsilon)$ is meant to be taken as a set-valued addition.  That is,
\begin{equation}
    F(\cdot) + B(0, \epsilon) = \{v + e : v \in F(\cdot), \|e\| \leq \epsilon\} .
\end{equation}

Under the given assumptions for $F$ in \eqref{eq:diffInc}, existence of a particular type of solution, a Carath\'eodory solution, may be guaranteed \cite{cortes2008}. A Carath\'eodory solution to \eqref{eq:diffInc} is an absolutely continuous trajectory $x$ : $[0, t_1] \to \mathbb{R}^n$ such that $\dot{x}(t) \in F(x(t))$ almost everywhere on $[0, t_{1}]$ and $x(0) = x_{0}$.  In this case, uniqueness can by no means be guaranteed.  For a more comprehensive survey of differential inclusions, see \cite{cortes2008discontinuous}.


Forward invariance in the context of differential inclusions typically admits two standard definitions, stemming from the nonuniqueness of solutions \cite{cortes2008discontinuous}. Specifically, weak invariance requires that at least one solution remains in the set for all time given an initial condition in $\mathcal{C}$.  On the other hand, strong invariance requests that given an initial condition in $\mathcal{C}$, all solutions stay in the set for all time as illustrated in Figure~\ref{fig:forward-invariance}. For the remainder of the paper, we will solely focus on strong invariance and simply refer to this quality as invariance.


\subsection{Barrier Functions for Differential Inclusions} 
\label{subsec:barrier-functions-for-differential-inclusions}

The work in \cite{PaulNBF} generalizes the result in \eqref{eq:barrier-certificate-reg} to nonsmooth barrier functions and differential inclusions.  However, in the case that the barrier function is smooth, the same result applies to a system described by a differential inclusion.  As such, this result becomes useful for this work, and it is subsequently stated in a form that has been modified to fit the terminology of this paper.  In the next section, we specialize this result for the purpose of robust control and apply it to validate CBFs for disturbed control systems.  Note that the results here are originally phrased for uncontrolled systems in \cite{PaulNBF}; however, for brevity, we still refer to barrier functions as CBFs, as the same results hold by considering the closed-loop system.

\begin{definition}{\cite[Definition~4]{PaulNBF}} 
    A locally Lipschitz function $h : \mathbb{R}^{n} \to \mathbb{R}$ is a valid \emph{Control Barrier Function (CBF)} for \eqref{eq:diffInc} if and only if $x_{0} \in \mathcal{C}$ implies that there exists a class-$\mathcal{KL}$ function $\beta : \mathbb{R}_{\geq 0} \times \mathbb{R}_{\geq 0} \to \mathbb{R}_{\geq 0}$ such that 
    \begin{equation}
        h(x(t)) \geq \beta(h(x_{0}), t), \forall t \in [0, t_{1}] ,
    \end{equation}
    for every Carath\'eodory solution $x : [0, t_{1}] \to \mathbb{R}^{n}$ starting from $x_{0}$.
\end{definition}

We will refer to functions $h$ satisfying the definition above as robust CBFs for the remainder of this paper.

\begin{theorem}{\cite[Theorem~2]{PaulNBF}}
    \label{thm:valid-cbf}
    Let $h : \mathbb{R}^{n} \to \mathbb{R}$ be a continuously differentiable function.  If there exists a locally Lipschitz extended class-$\mathcal{K}$ function $\alpha: \mathbb{R} \to \mathbb{R}$ such that 
    \begin{equation}
        \min \nabla h(x')^{\top}F(x') \geq - \alpha(h(x')), \forall x' \in \mathbb{R}^n,
    \end{equation}
    then $h$ is a valid robust CBF for \eqref{eq:diffInc}.
\end{theorem}
\todo{Note that the term $\nabla h(x')^{\top}F(x')$ in Theorem~\ref{thm:valid-cbf} contains a set-valued multiplication, which is defined as 
\begin{equation}
   \nabla h(x')^{\top}F(x') = \{\nabla h(x')^{\top} f : f \in F(x')\};
\end{equation}
moreover, the minimum from Theorem~\ref{thm:valid-cbf} is guaranteed to exist since $F$ is assumed to take nonempty and compact values.}

\begin{remark}
    In \cite{PaulNBF}, the gradient $\nabla h$ is replaced with the Clarke generalized gradient $\partial_{c} h : \mathbb{R}^{n} \to 2^{\mathbb{R}^{n}}$ (see \cite{clarke1990}) and $h$ is only assumed to be locally Lipschitz continuous.  However, in the case that $h$ is continuously differentiable, these two objects are equivalent.  Hence, Theorem~\ref{thm:valid-cbf} is equivalent to \cite[Theorem~2]{PaulNBF} in the context of this work.
\end{remark}

\begin{figure}
\centering
\includegraphics[height=0.41\linewidth]{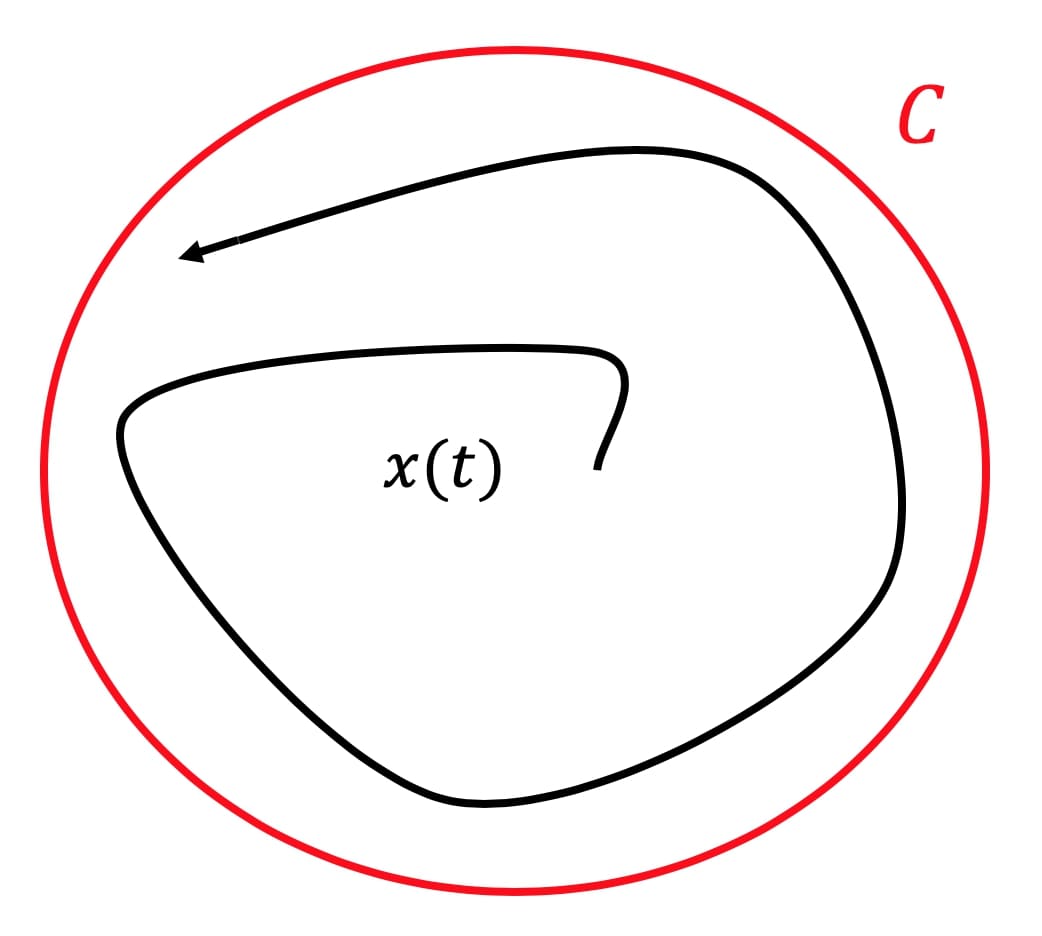} \hspace{3mm}
\includegraphics[height=0.41\linewidth]{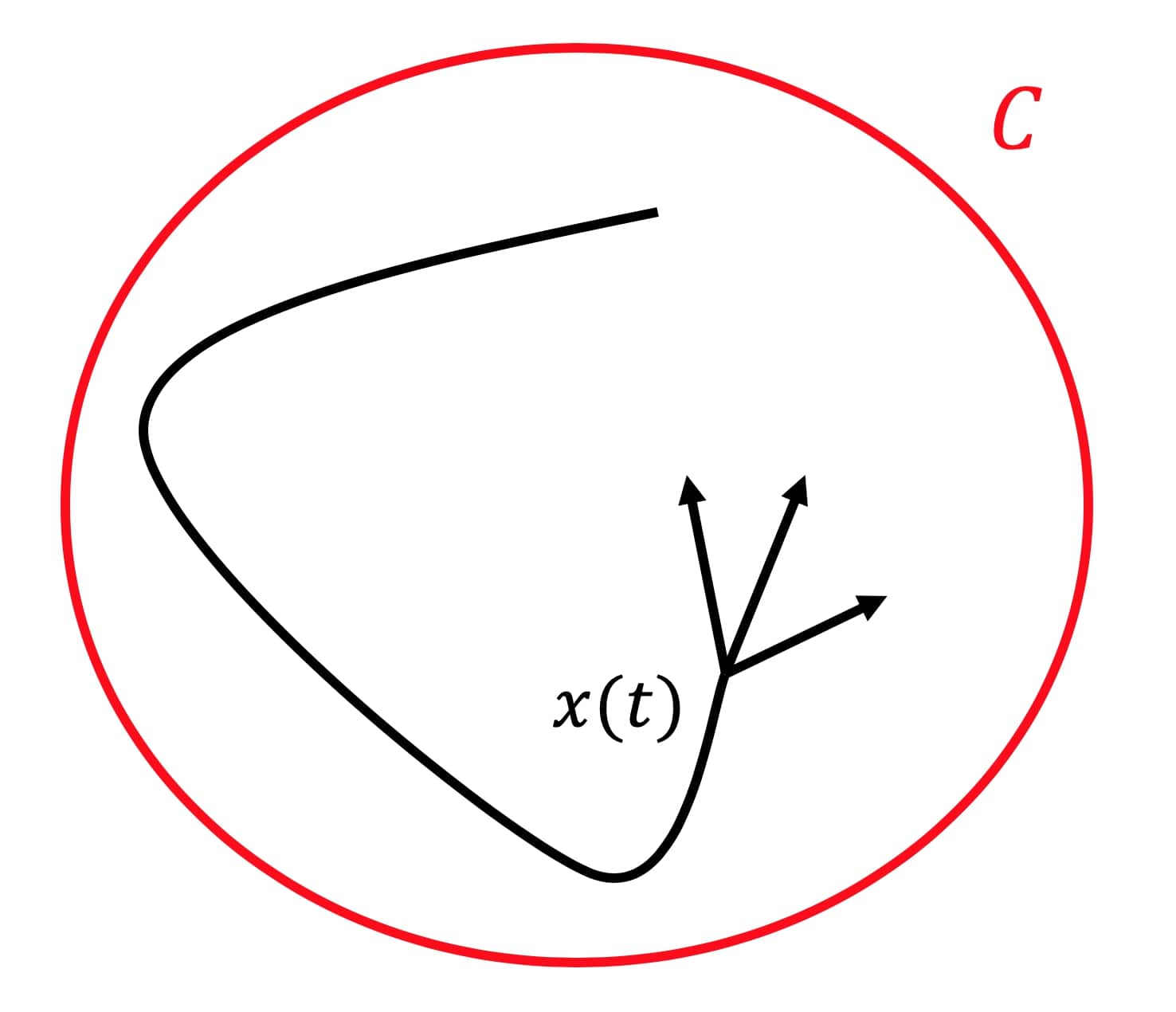}
\caption{Forward invariance of set $\mathcal{C}$ for systems modelled by ordinary differential equations (left) and differential inclusions (right). Note that in the case of differential inclusions, strong invariance indicates that all possible trajectories remain within the set $\mathcal{C}$.}
\label{fig:forward-invariance}
\end{figure}

\section{Barrier Functions for Disturbed Dynamical Systems}
\label{sec:barrier-functions-for-disturbed}
This section introduces a general approach that efficiently accounts for a large class of disturbances in dynamic models. Given the control-affine system \eqref{eq:control-affine}, the following two subsections will address disturbances on $f(x)$ and $g(x)$ respectively. Note that although we solely address these specific types of disturbances in the context of CBFs, the general approach can be applied to other frameworks (e.g., CLFs for guaranteeing stability \cite{galloway2015torque}). \todo{Moreover, although not addressed in this paper, the results presented in the remainder of this section can be straightforwardly extended to account for time-varying disturbances.}      

\subsection{Additive Disturbance}
\label{subsec:cbf_additive}
This subsection presents theoretical results from \cite{emam2019robust}, where the authors focus on disturbed control-affine systems that can be modelled through the following differential inclusion
\begin{equation} 
    \label{eq:control-affine-disturbed}
    \dot{x}(t) \in f(x(t)) + g(x(t))u(x(t)) + D_{A}(x(t)), x(0) = x_{0} ,
\end{equation}
where $D_A : \mathbb{R}^{n} \to 2^{\mathbb{R}^{n}}$ (the disturbance) is an upper semi-continuous set-valued map that takes nonempty, convex, and compact values; and $f$, $g$, $u$ are as in \eqref{eq:control-affine}. \todo{We present the following corollary that specializes the results from Theorem~\ref{thm:valid-cbf} to the dynamical system described by \eqref{eq:control-affine-disturbed}}. 

\todo{
\begin{corollary}
    \label{cor:disturbed-cbf}
    Let $h : \mathbb{R}^n \to \mathbb{R}$ be a continuously differentiable function.  If there exists a continuous function $u : \mathbb{R}^{n} \to \mathbb{R}^{m}$ and a locally Lipschitz extended class-$\mathcal{K}$ function $\alpha: \mathbb{R} \to \mathbb{R}$ such that 
    \begin{align}
        \begin{split}
            \min \nabla h(x')^{\top}(f(x') + g(x')u(x') + D_A(x')) \geq & \\
            - \alpha(h(x')), \forall x' \in \mathbb{R}^n &,
        \end{split}
    \end{align}
    then $h$ is a valid CBF for \eqref{eq:control-affine-disturbed}.
\end{corollary}}

We assume that $D_A$ can be represented as a  convex hull of $p > 0$ continuous functions $\psi_{i} : \mathbb{R}^{n} \to \mathbb{R}^{n}$, $i \in \{1, \hdots, p\}$ as
\begin{equation} 
    \label{eq:defPsi}
    D_A(x') = \co\Psi(x') = \co\{\psi_1(x')\ldots\psi_p(x')\}, \forall x' \in \mathbb{R}^{n} .
\end{equation}

\begin{lemma}
    \label{lem:usc-disturbance}
    Let $\psi_{i} : \mathbb{R}^{n} \to \mathbb{R}^{n}$, $i \in \{1, \hdots, p\}$ be a set of $p > 0$ continuous functions.  Then, $D_A : \mathbb{R}^{n} \to 2^{\mathbb{R}^{n}}$ defined as 
    \begin{equation}
        D_A(x') = \co \Psi(x') = \co \{\psi_{i}(x') : i \in \{1, \hdots, p\}\}, \forall x' \in \mathbb{R}^{n} ,
    \end{equation}
    is an upper semi-continuous set-valued map that takes compact, nonempty, and convex values.
\end{lemma}

\begin{proof}
    Let $x' \in \mathbb{R}^{n}$.  Then, by definition, $D_A(x')$ is convex and nonempty, since $p > 0$.  Moreover, $D_A(x')$ is compact, since it is the convex hull of a finite number of points.  
    
    Now, it remains to show upper semi-continuity.  Let $\epsilon > 0$.  Because each $\psi_{i}$ is continuous, there exists corresponding $\delta_{i} > 0$ such that 
    \begin{equation}
        \|\psi_{i}(y) - \psi_{i}(x')\| \leq \epsilon, \forall y \in B(x', \delta_{i}) ,
    \end{equation}
    meaning that 
    \begin{equation}
        \psi_{i}(y) \in \psi_{i}(x') + B(0, \epsilon), \forall y \in B(x', \delta_{i}) .
    \end{equation}
    Set 
    \begin{equation}
        \delta = \min_{i} \delta_{i} ,
    \end{equation}
    which satisfies $\delta > 0$, because $p$ is finite.  Since
    \begin{equation}
    \label{eq:disturbance}
        D_A(\cdot) = \co \{\psi_{i}(\cdot) : i \in \{1, \hdots, p\}\} , 
    \end{equation}
    it now suffices to show that 
    \begin{align}
        \Psi(y) \subset \Psi(x') + B(0, \epsilon), \forall y \in B(x', \delta) .
    \end{align} 
    Let $i \in \{1, \hdots, p\}$ and consider $\psi_{i}(y)$.  By choice of $\delta$, 
    \begin{equation}
        \psi_{i}(y) \in \psi_{i}(x') + B(0, \epsilon) ,
    \end{equation}
    as such 
    \begin{equation}
        \psi_{i}(y) \in \Psi(x') + B(0, \epsilon) .
    \end{equation}
    Accordingly, $D_A$ is upper semi-continuous.
\end{proof}

Note that we will later show that this assumption is highly non-restrictive. Moreover, by leveraging the latter assumption along with Theorem~\ref{thm:valid-cbf}, one can efficiently ensure the validity of a CBF for the disturbed dynamical system \eqref{eq:control-affine-disturbed} by solely checking for the extreme points of the set $D_A$. The following theorem from \cite{emam2019robust} formalizes this result. 
\begin{theorem} 
    \label{thm:valid-cbf-disturbed}
    Let $h : \mathbb{R}^{n} \to \mathbb{R}$ be a continuously differentiable function.  Let $\psi_{i} : \mathbb{R}^{n} \to \mathbb{R}^{n}$, $i \in \{1, \hdots, p\}$ be a set of $p > 0$ continuous functions, and define the disturbance $D_A : \mathbb{R}^{n} \to 2^{\mathbb{R}^{n}}$ as 
    \begin{equation}
        D_A(x') = \co \Psi(x') = \co\{\psi_1(x')\ldots\psi_p(x')\}, \forall x' \in \mathbb{R}^{n} .
    \end{equation}
    If there exists a continuous function $u : \mathbb{R}^{n} \to \mathbb{R}^{m}$ and a locally Lipschitz extended class-$\mathcal{K}$ function $\alpha : \mathbb{R} \to \mathbb{R}$ such that 
    \begin{equation} 
        \label{eq:mainTheorem}
            \begin{split}
                & \nabla h(x')^{\top}(f(x') + g(x')u(x')) \geq \\
                & -\alpha(h(x')) - \min \nabla h(x')^{\top} \Psi(x'), \forall x' \in \mathbb{R}^{n} ,
            \end{split}
    \end{equation}
    then $h$ is a valid robust CBF for \eqref{eq:control-affine-disturbed}.
\end{theorem}
\begin{proof}
    \label{proof:valid-cbf-disturbed}
       \todo{By substituting the definition of $D_A(x')$ from \eqref{eq:defPsi} into the result from Corollary~\ref{cor:disturbed-cbf}, it follows that it must be shown that}
    \begin{align}
        \begin{split}
            & \min \nabla h(x')^{\top}(f(x') + g(x')u(x') + \co \Psi(x')) \geq \\
            & - \alpha(h(x')), \forall x' \in \mathbb{R}^{n} .
        \end{split}
    \end{align}
    \todo{By Lemma~\ref{lem:usc-disturbance}, $\co \Psi$ is an upper semi-continuous set-valued map that takes nonempty, convex, and compact values, so the results of Corollary~\ref{cor:disturbed-cbf} may be applied.} Note that, for any $x' \in \mathbb{R}^{n}$, the condition above is equivalent to
    \begin{align} 
        \begin{split}
            \label{eq:ineq1}
            & \nabla h(x')^{\top}(f(x') + g(x')u(x')) \geq \\
            & -\alpha(h(x')) - \min \nabla h(x')^{\top} \co \Psi(x') .
        \end{split}
    \end{align}
    We can then take advantage of the properties of the convex hull (see \cite[Lemma~3]{PaulNBF}) through the following equality
    \begin{equation} 
        \label{eq:minCo}
        \min \nabla h(x')^{\top} \text{co} \Psi(x') = \min \nabla h(x')^{\top} \Psi(x') .
    \end{equation}
    Thus, by substituting \eqref{eq:minCo} into \eqref{eq:ineq1}, we obtain \eqref{eq:mainTheorem}.
\end{proof}

Note that checking every disturbance in $\co \Psi(\cdot)$ is equivalent to only the extreme points of $\co \Psi(\cdot)$ (i.e., each $\psi_{i}(\cdot)$) which has a linear computational cost with respect to the size of the set $\Psi(x')$. 
\todo{Moreover, Theorem~\ref{thm:valid-cbf-disturbed} can be straightforwardly extended to the case where $D_A$ is the union of the convex hulls of a finite number of function-valued points, as shown by the following proposition.

\begin{proposition}
    \label{prop:union-of-convex-hulls}
    Define 
    \begin{equation} 
        \label{eq:Union}
        D_A(x') = \bigcup_{i=1}^{q} \co(\Psi_{i}(x')), \forall x' \in \mathbb{R}^{n}
    \end{equation}
    such that
    \begin{equation}
        \Psi_{i}(x') = \{\psi^i_{1}(x'), \hdots, \psi^{i}_{p_{i}}(x')\}, \forall i \in \{1, \hdots, q\}, 
    \end{equation}
    where each $\psi^{i}_{j}(\cdot)$ is continuous. Let $h : \mathbb{R}^{n} \to \mathbb{R}$ be a continuously differentiable function.  If there exists a continuous function $u : \mathbb{R}^{n} \to \mathbb{R}^{m}$ and a locally Lipschitz extended class-$\mathcal{K}$ function $\alpha : \mathbb{R} \to \mathbb{R}$ such that 
    \begin{align} 
        \begin{split}
            & \nabla h(x')^{\top}(f(x') + g(x')u(x')) \geq \\
            & -\alpha(h(x')) - \min \nabla h(x')^{\top} \Psi_i(x'), \forall i \in \{1, \hdots, q\} ,
        \end{split}
    \end{align}
    then $h$ is a valid CBF for \eqref{eq:control-affine-disturbed}.
\end{proposition}
\begin{proof} 
    The proposition is directly obtained by applying Theorem~\ref{thm:valid-cbf-disturbed} $q$ times.  Given that inequality \eqref{eq:mainTheorem} holds for each $D_i(x') = \co \Psi_i(x')$ then it follows that it holds for $ D_A(x')~=~\bigcup_{i=1}^{q} D_i(x')$.
\end{proof}
\begin{remark}
    \label{rem:magic-method}
    The proof above illustrates the fact that $h$ is a valid CBF for 
    \begin{equation}
        D_A(x') = \bigcup_{i=1}^{q} \co \Psi_{i}(x'), \forall x' \in \mathbb{R}^{n} 
    \end{equation}
    if and only if it is also valid for
    \begin{equation}
        D_A(x') = \co \bigcup\limits_{i=1}^{q} \Psi_{i}(x'), \forall x' \in \mathbb{R}^{n}.
    \end{equation}
This is because in both cases, all the extreme points of $\bigcup_{i=1}^{q} \Psi_{i}(x')$ are accounted for: either through multiple constraints as in Proposition~\ref{prop:union-of-convex-hulls} for $D_A(x') = \bigcup_{i=1}^{q} \co \Psi_{i}(x')$ or through a single constraint as in Theorem~\ref{thm:valid-cbf-disturbed} for $D_A(x')~=~\co~\bigcup_{i=1}^{q} \Psi_{i}(x')$.
\end{remark}

The nonconvexity of $D_A$, as given in Proposition~\ref{prop:union-of-convex-hulls}, may appear to pose a problem, as the sufficient conditions for the existence of solutions to a differential inclusion requires that the set-valued map takes convex values (see \eqref{eq:diffInc}).  However, as noted by Remark~\ref{rem:magic-method}, $D_A$ may be equivalently defined as 
\begin{equation}
    D_A(x') = \co \bigcup\limits_{i=1}^{q} \Psi_{i}(x'), \forall x' \in \mathbb{R}^{n} ,
\end{equation}
which is indeed convex.} Describing disturbances as a finite union of convex sets encodes a very wide class of disturbances.  However, Remark~\ref{rem:magic-method} implies that considering this nonconvex disturbance is actually completely equivalent to considering the convex hull of the disturbance.  As such, this result indicates that utilizing convex hulls to approximate a disturbance equivalently addresses a wide class of nonconvex disturbances.


\subsection{Multiplicative Disturbance}
\label{subsec:cbf_mult}
In this subsection, we present results for the case where the disturbance affects the value of $g(x)$ in \eqref{eq:control-affine}, and thus it multiplies the control. Consider the following dynamical system
\begin{equation} 
    \label{eq:control-affine-disturbed-mult}
    \dot{x}(t) \in f(x(t)) + (g(x(t)) + D_M(x(t)))u(x(t)), x(0) = x_{0}, 
\end{equation}
where $D_M : \mathbb{R}^{n} \to 2^{\mathbb{R}^{n \times m}}$ is an upper semi-continuous set-valued map that takes nonempty, convex, and compact values. Similarly to the previous subsection, we also assume that the disturbance is a convex hull of $p > 0$ continuous functions $i \in \{1, \hdots, p\}$ as
\begin{equation} 
    \label{eq:defPsi-mult}
    D_M(x') = \co\Psi(x') = \co\{\psi_1(x')\ldots\psi_p(x')\}, \forall x' \in \mathbb{R}^{n} .
\end{equation}

\todo{
\begin{lemma}
    \label{lem:usc-disturbance-m}
    Let $\psi_{i} : \mathbb{R}^{n} \to \mathbb{R}^{n \times m}$, $i \in \{1, \hdots, p\}$ be a set of $p > 0$ continuous functions.  Then, $D_M : \mathbb{R}^{n} \to 2^{\mathbb{R}^{n\times m}}$ defined as 
    \begin{equation}
        D_M(x') = \co \Psi(x') = \co \{\psi_{i}(x') : i \in \{1, \hdots, p\}\}, \forall x' \in \mathbb{R}^{n} ,
    \end{equation}
    is an upper semi-continuous set-valued map that takes compact, nonempty, and convex values.
\end{lemma}}
\noindent
The following corollary \todo{of Theorem~\ref{thm:valid-cbf}} states a sufficient condition for a function $h$ to be a valid robust CBF for \eqref{eq:control-affine-disturbed-mult}.

\begin{corollary}
    \label{prop:disturbed-cbf-mult}
    Let $h : \mathbb{R}^n \to \mathbb{R}$ be a continuously differentiable function.  If there exists a continuous function $u : \mathbb{R}^{n} \to \mathbb{R}^{m}$ and a locally Lipschitz extended class-$\mathcal{K}$ function $\alpha: \mathbb{R} \to \mathbb{R}$ such that 
    \begin{align}
    \label{eq:multiplicative-1}
        \begin{split}
            & \nabla h(x')^{\top}(f(x') + g(x')u(x')) \\
            & + \min (\nabla h(x')^{\top} D_M(x') u(x')) \geq  -\alpha(h(x'))  , \forall x' \in \mathbb{R}^{n}.
        \end{split}
    \end{align}
    then $h$ is a valid robust CBF for \eqref{eq:control-affine-disturbed}.
\end{corollary}

Note that as opposed to the additive disturbance case, the term $\min (\nabla h(x')^{\top}  D_M(x') u(x'))$ is non-linear with respect to the control signal. However, similarly to \cite{emam2019robust}, we wish to embed this constraint into a real-time controller synthesis framework---which requires the control input to be generated at a frequency greater than $100$Hz---via Quadratic Programming (QP).  In turn, QPs require the inequality constraints to be affine with respect to $u(x')$. To achieve this requirement, \eqref{eq:multiplicative-1} may be equivalently represented as $p$ constraints as presented in the following theorem.   

\begin{theorem} 
\label{thm:valid-cbf-disturbed-m}
Let $h : \mathbb{R}^{n} \to \mathbb{R}$ be a continuously differentiable function.  Let $\psi_{i} : \mathbb{R}^{n} \to \mathbb{R}^{n \times m}$, $i \in \{1, \hdots, p\}$ be a set of $p > 0$ continuous functions, and define the disturbance $D_M : \mathbb{R}^{n} \to 2^{\mathbb{R}^{n \times m}}$ as 
\begin{equation}
    D_M(x') = \co \Psi(x') = \co\{\psi_1(x')\ldots\psi_p(x')\}, \forall x' \in \mathbb{R}^{n} .
\end{equation}
If there exists a continuous function $u : \mathbb{R}^{n} \to \mathbb{R}^{m}$ and a locally Lipschitz extended class-$\mathcal{K}$ function $\alpha : \mathbb{R} \to \mathbb{R}$ such that 
\begin{equation} 
    \label{eq:mainTheorem-m}
        \begin{split}
            & \nabla h(x')^{\top}(f(x') + (g(x') +  \psi_i(x')) u(x')) \geq \\
            & -\alpha(h(x')), \forall x' \in \mathbb{R}^{n} , \forall i \in \{1,\ldots, p\},
        \end{split}
\end{equation}
then $h$ is a valid robust CBF for \eqref{eq:control-affine-disturbed-mult}.
\end{theorem}

\begin{proof}
    \label{proof:valid-cbf-disturbed-mult}
    We begin by substituting the definition of $D_M(x')$ from \eqref{eq:defPsi-mult} into \eqref{eq:multiplicative-1}.  In particular, we need to show that
    \begin{equation}
        \begin{aligned}
            & \nabla h(x')^{\top}(f(x') + g(x')u(x')) \\
            & + \min (\nabla h(x')^{\top} \co \Psi(x') u(x')) \geq  -\alpha(h(x'))  , \forall x' \in \mathbb{R}^{n}.
        \end{aligned}
        \label{eq:ineq1}
    \end{equation}
   Note that by \todo{Lemma~\ref{lem:usc-disturbance-m}}, $D_M$ as defined above satisfies the required properties of \eqref{eq:control-affine-disturbed-mult}. We leverage the following property of convex hulls
\begin{equation}
\begin{aligned} 
    \min \nabla h(x')^{\top} \text{co} \Psi(x') u(x') &= \min \co \nabla h(x')^{\top} \Psi(x') u(x') \\
    &= \min \nabla h(x')^{\top} \Psi(x') u(x').
\end{aligned}
\label{eq:minCo}
\end{equation}
    Substituting \eqref{eq:minCo} into \eqref{eq:ineq1} yields 
    \begin{align}
        \begin{split}
        \label{eq:ineq2}
            & \nabla h(x')^{\top}(f(x') + g(x')u(x')) \\
            & + \min (\nabla h(x')^{\top} \Psi(x') u(x')) \geq  -\alpha(h(x'))  , \forall x' \in \mathbb{R}^{n},
        \end{split}
    \end{align}    
    which is equivalent to \eqref{eq:mainTheorem-m}.
\end{proof}

Moreover, similarly to the additive disturbance case, accounting for disturbances $D_M$ modelled as unions of convex hulls  follows straightforwardly from Theorem~\ref{thm:valid-cbf-disturbed-m}.
\begin{remark}
    \label{rem:magic-method-m}
    The function $h$ is a valid robust CBF for 
    \begin{equation}
        D_M(x') = \bigcup_{i=1}^{q} \co \Psi_{i}(x'), \forall x' \in \mathbb{R}^{n}
    \end{equation}
    if and only if it is also valid for
    \begin{equation}
        D_M(x') = \co \bigcup\limits_{i=1}^{q} \Psi_{i}(x'), \forall x' \in \mathbb{R}^{n}.
    \end{equation}
\end{remark}

\subsection{Matrices of convex hulls}
In addition to Theorems~\ref{thm:valid-cbf-disturbed}~and~\ref{thm:valid-cbf-disturbed-m}, a common case when dealing with data-driven disturbance methods is to obtain a disturbance set whose entries are convex hulls, which is a different representation from the one assumed in the theorems (convex hulls of a finite set of points). In this subsection, we show how \todo{the proposed} framework can be extended to account for such cases.   
\subsubsection{Additive Disturbance}
In the additive disturbance case, the aforementioned disturbance set is defined as  $D_{A} : \mathbb{R}^{n} \to 2^{\mathbb{R}^{n}}$ whose entries are convex hulls
\begin{equation}
\label{eq:d_est-a}
    [D_{A}(x')]_{i} = \co\{a^1_{i}(x'), a^2_{i}(x')\} = [a^1_{i}(x'), a^2_{i}(x')],
\end{equation}
where $a^1_{i}: \mathbb{R}^{n} \to \mathbb{R}$ and $a^2_{i}: \mathbb{R}^{n} \to \mathbb{R}$ are continuous functions.  Note that this is a different representation than \eqref{eq:defPsi} used by Theorem~\ref{thm:valid-cbf-disturbed} where the disturbance is the convex hull of a set of points. A straightforward way to obtain the desired form is by defining $\Psi$ as the $2^{n}$ vectors generated by permuting the entries of $D_{A}$ from \eqref{eq:d_est-a}. For example, consider the case where $n = 2$, then the corresponding disturbance ${D}_{A}$ would be defined as
\begin{equation*}
\begin{split}
D_{A}(x') = \co \left\{ 
\begin{bmatrix}
a^1_{1}(x') \\
a^1_{2}(x')
\end{bmatrix},
\begin{bmatrix}
a^1_{1}(x') \\
a^2_{2}(x')
\end{bmatrix},
\begin{bmatrix}
a^2_{1}(x') \\
a^1_{2}(x')
\end{bmatrix},
\begin{bmatrix}
a^2_{1}(x') \\
a^2_{2}(x')
\end{bmatrix}
\right\}.
\end{split}
\end{equation*}

\subsubsection{Multiplicative Disturbance}
In the case of the multiplicative disturbance, the disturbance consisting of matrix whose entries are convex hulls is defined as $D_{M} : \mathbb{R}^{n} \to 2^{\mathbb{R}^{n \times m}}$ where
\begin{equation}
\label{eq:d_est-m}
    [D_{M}(x')]_{ij} = \co\{a^1_{ij}(x'), a^2_{ij}(x')\} = [a^1_{ij}(x'), a^2_{ij}(x')], 
\end{equation}
and $a^1_{ij}: \mathbb{R}^{n} \to \mathbb{R}$ and $a^2_{ij}: \mathbb{R}^{n} \to \mathbb{R}$ are continuous functions.  Note that as opposed to the additive disturbance case, the permutation of the entries of $D_{M}$ would result in $2^{nm}$ matrices. As such, a direct application of Theorem~~\ref{thm:valid-cbf-disturbed-m} would yield $2^{nm}$ constraints, which may be a prohibitively large number for some systems. For example, in the case of a differential-drive robot with $n = 3$ and $m = 2$, this would result in $64$ constraints. To remedy this, we state the following proposition that allows us to reduce the number of constraints to $2^m$ (which corresponds to only $4$ constraints for the differential-drive robot example).

\begin{proposition} 
\label{prop:valid-cbf-disturbed-m2}
Let $h : \mathbb{R}^{n} \to \mathbb{R}$ be a continuously differentiable function.  Let $\Phi(x') = \{\phi_i(x')\}^{2^m}_{i=1}$ be the set of vertices of the $m$-orthotope defined by 
$$\nabla h(x')^{\top} D_{M}(x'),$$
where $D_{M}(x')$ is defined as in \eqref{eq:d_est-m}. 
If there exists a continuous function $u : \mathbb{R}^{n} \to \mathbb{R}^{m}$ and a locally Lipschitz extended class-$\mathcal{K}$ function $\alpha : \mathbb{R} \to \mathbb{R}$ such that 
\begin{equation} 
    \label{eq:mainTheorem-m2}
        \begin{split}
            & \nabla h(x')^{\top}(f(x') + g(x') u(x')) + \phi_i(x') u(x')\geq \\
            & -\alpha(h(x')), \forall x' \in \mathbb{R}^{n} , \forall i \in \{1,\ldots, 2^m\},
        \end{split}
\end{equation}
then $h$ is a valid robust CBF for \eqref{eq:control-affine-disturbed-mult}.
\end{proposition}

\begin{figure}
\centering
\includegraphics[width=\linewidth]{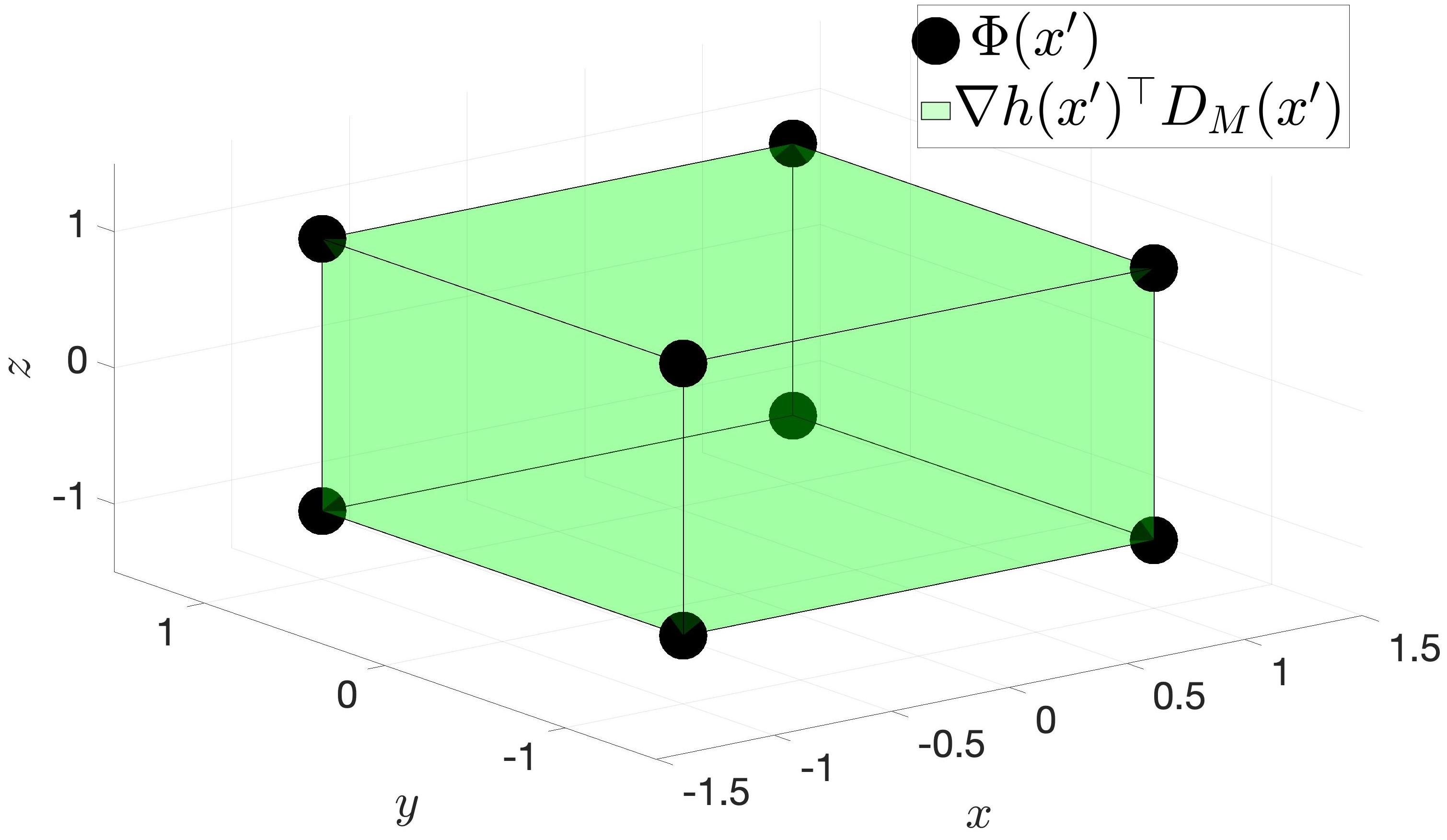}
\caption{\todo{An arbitrary example of $\nabla h(x')^{\top}D_{M}(x')$ for $m = 3$. The set $\Phi(x') = \{\phi_i(x')\}^{2^m}_{i=1}$ used in Proposition~\ref{prop:valid-cbf-disturbed-m2} is the set of vertices denoted by the black dots of the $m$-orthotope defined by $\nabla h(x')^{\top}D_{M}(x')$ shown in green.}}
\label{fig:d-est-example}
\end{figure}

\begin{proof}
We need to show that 
\begin{equation}
        \begin{split}
        \label{eq:ineq2}
            & \nabla h(x')^{\top}(f(x') + g(x')u(x')) \\
            & + \min (\nabla h(x')^{\top} D_{M}(x') u(x')) \geq  -\alpha(h(x'))  , \forall x' \in \mathbb{R}^{n},
        \end{split}
\end{equation}
is equivalent to \eqref{eq:mainTheorem-m2}. We begin by noting that 
\begin{equation}
\label{eq:minCo-2}
\nabla h(x')^{\top} D_{M}(x') = \co \Phi(x').
\end{equation}
Using \eqref{eq:minCo-2}, we obtain 
\begin{equation}
    \label{eq:minCo-3}
    \begin{split}
    \min (\nabla h(x')^{\top} D_{M}(x') u(x')) &= \min (\co \Phi(x') u(x')) \\
    &= \min (\Phi(x') u(x')).
    \end{split}
\end{equation}
Next, we substitute \eqref{eq:minCo-3} into \eqref{eq:ineq2} which yields 
    \begin{align}
        \begin{split}
        \label{eq:ineq1-2}
            & \nabla h(x')^{\top}(f(x') + g(x')u(x')) \\
            & + \min (\Phi(x') u(x')) \geq  -\alpha(h(x'))  , \forall x' \in \mathbb{R}^{n} ,
        \end{split}
    \end{align}    
    which is equivalent to \eqref{eq:mainTheorem-m2}.
\end{proof}


Figure~\ref{fig:d-est-example} pictorially demonstrates how the set $\Phi(x')$ used in the proposition above is defined with respect to $\nabla h(x')^{\top}D_{M}(x')$. Note that the m-orthotope is constructed as follows
\begin{equation}
    \begin{split}
        [\nabla h(x')^{\top}D_{M}(x')]_j =  \co\{\sum_i \min_k[\nabla h(x')]_i a^k_{ij}, \\ \sum_i \max_k[\nabla h(x')]_i a^k_{ij}\},
    \end{split}
\end{equation}
where $a^k_{ij}$ is defined as in \eqref{eq:d_est-m}.

\subsection{Discussion}

The convenient structure of the conditions presented in Theorems~\ref{thm:valid-cbf-disturbed}~and~\ref{thm:valid-cbf-disturbed-m} allow for the synthesis of controllers that satisfy \eqref{eq:mainTheorem}~or~\eqref{eq:mainTheorem-m} for a given system in real time, since only a finite number of points need to be evaluated. In  section~\ref{sec:control-synthesis-via}, we present a QP for controller synthesis, which we apply to the Robotarium as presented in section~\ref{sec:robust-collision-avoidance}. However, one question that remains is: how does one estimate the disturbances $D_A$ and $D_M$? An advantage of this approach is that any finite-cardinality set satisfies the theoretical requirements. As such, one straightforward method of estimating the disturbance involves gathering a large number of data points then computing $\Psi$ via \eqref{eq:control-affine-disturbed} or \eqref{eq:control-affine-disturbed-mult}, and fitting a convex set to the data. An example of the latter is shown in Figure~\ref{fig:psi_estimation}, where the additive disturbance on the state of a differential-drive robot is estimated using a bounding box including $> 99\%$ of the datapoints obtained via a data collection process similar to the one described in Section~\ref{sec:disturb-estimation}. Another scenario is when the dynamic model is  affected by a disturbance with a known distribution. For example, consider the following system 
\begin{equation*}
    \dot{x}(t) = f(x(t)) + g(x(t))u(x(t)) + \epsilon
\end{equation*}
where $\epsilon \sim \mathcal{N}(\mu, \Sigma)$ is Gaussian noise having a mean and diagonal covariance matrix (i.e., uncorrelated noise) known apriori. In this case, we can straightforwardly account for this noise by leveraging the additive disturbance CBF formulation and defining 
$$[D_A]_i = \mu_i + [-k_c \sigma_i, k_c \sigma_i], \forall i \in \{1,\ldots,n\},$$
where $\mu_i$, $\sigma_i^2$ denote the $i$-th component of the mean vector and covariance matrix's diagonal, respectively. The coefficient $k_c$ is a user-chosen confidence parameter (e.g., $k_c~=~2$ achieves a confidence of $95.5\%$). Note that this can be extended to the case where the covariance matrix is not diagonal.
However, in many cases, the disturbance could vary spatially (e.g., uneven terrain, wind), and we would like our estimation to capture this. Moreover, it may not be possible to obtain the data points or disturbance distribution apriori; in these cases, the disturbance needs to be estimated online.  Therefore, in the next section, we present an online, sample-efficient disturbance-estimation method based on Gaussian processes.  

\begin{figure}
    \centering
    \includegraphics[width=\linewidth]{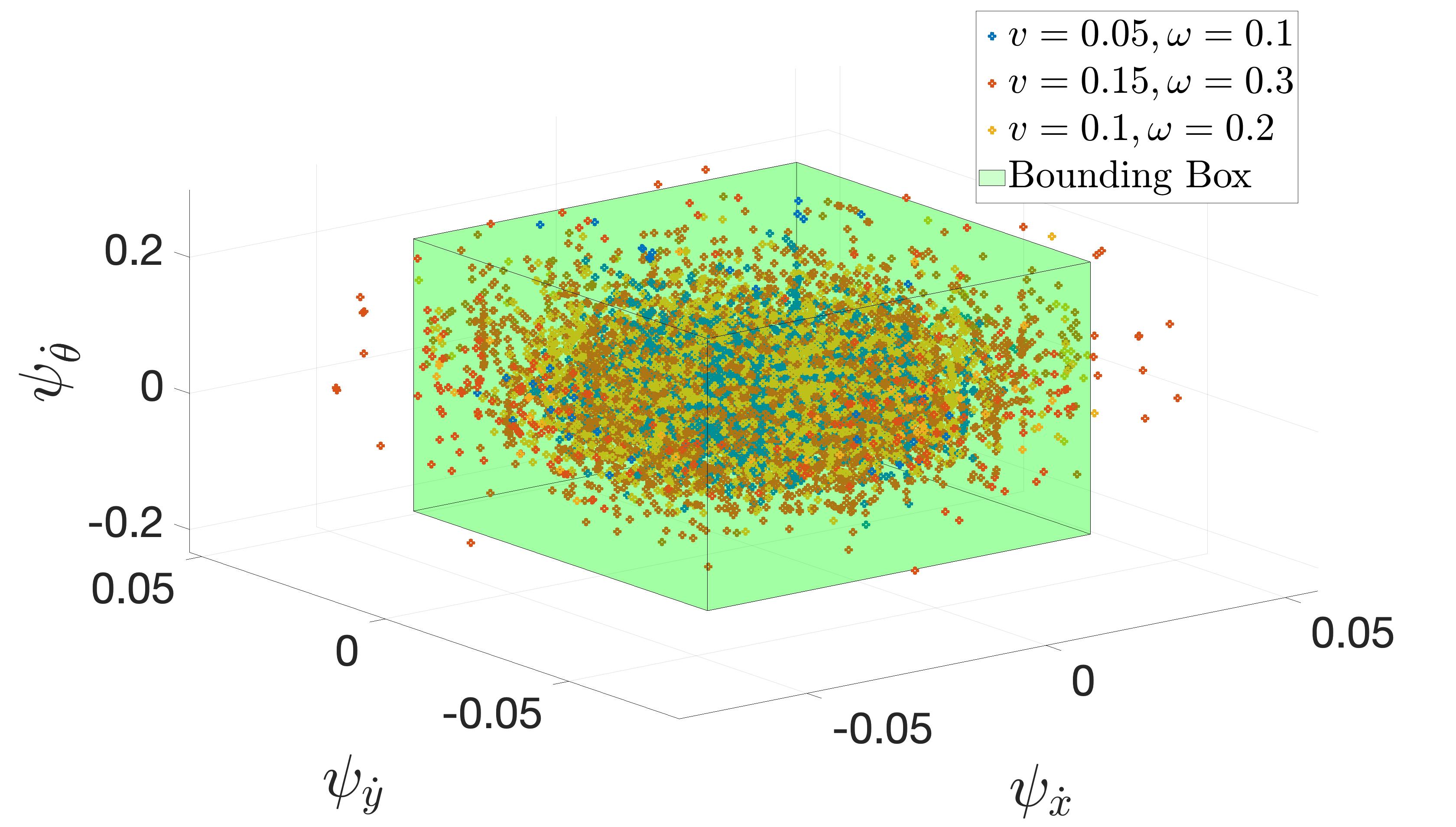}
    \caption{Additive disturbance estimation example for a differential drive robot through a bounding box that includes $>99\%$ of the collected datapoints. The symbols $\psi_{\dot{x}}$, $\psi_{\dot{y}}$ and $\psi_{\dot{\theta}}$  denote the disturbance on  $\dot{x}$, $\dot{y}$ and $\dot{\theta}$ respectively.}
    \label{fig:psi_estimation}
\end{figure}

\section{Estimating the Disturbance Set through Gaussian Process Regression}
\label{sec:disturb-estimation}

In this section we demonstrate how the disturbance sets introduced in the previous section can be estimated using  Gaussian Process Regression (GPR). GPR \cite{rasmussen2003gaussian} is a kernel-based regression model used for prediction in many applications such as robotics, reinforcement learning, and data visualisation \cite{deisenroth2015distributed}. One of the main advantages of GPR compared to other regression methods is data efficiency, which allows us to obtain a good estimate of the disturbance using only a small number of data-points. We begin by introducing the general problem setup for GPR and then proceed to specialize the results for our application.

\subsection{Problem Formulation}
\label{subsec:gpr_intro}
\todo{Given training data $\mathcal{D} = \{x^{(i)},y^{(i)}\}_{i=1}^N$ where $x^{(i)} \in \mathbb{R}^n$ and $y^{(i)} \in \mathbb{R}$, the objective is to approximate the unknown target function $f_D: \mathbb{R}^n \rightarrow \mathbb{R}$ which maps an input $x^{(k)}$ to the target value $y^{(i)}$ given the model $y^{(i)} = f_D(x^{(i)}) + \epsilon^{(i)}$, where $\epsilon^{(i)} \sim \mathcal{N}(0,\sigma_n^2)$ is Gaussian noise having zero mean and variance $\sigma^2_n$. Thus, the prior on the target values can be described by \mbox{$\mathbf{y} \sim \mathcal{N}(0,\mathbf{K} + \sigma_n^2   \mathbf{I})$ where $\mathbf{K}$} is the covariance matrix computed through a covariance function $k: \mathbb{R}^n~\times~\mathbb{R}^n \rightarrow~\mathbb{R}$; one commonly used choice of which is the Gaussian kernel given by 
\begin{align}
[\mathbf{K}]_{ij} &= k(x^{(i)},x^{(j)}) \\
&= \sigma_s^2 \exp{ \left(- \frac{1}{2} (x^{(i)} - x^{(j)})^{\top} W (x^{(i)} - x^{(j)})\right) },
\end{align}
where $\sigma_s$ and $W$ denote the signal variance and the kernel widths respectively. Note that the choice of a zero-mean prior is motivated by the fact that we only estimate the disturbance term in the dynamics which we assume to have zero-mean. As such, the joint distribution of the training labels and predicted value for a query point $x_{\ast}$  is given by
$$
\begin{bmatrix}
\mathbf{y} \\ 
f_D(x_{\ast})
\end{bmatrix} = 
\mathcal{N}\left(\mathbf{0}, 
\begin{bmatrix}
\mathbf{K} + \sigma_n^2\mathbf{I} & \mathbf{k}_{\ast} \\
\mathbf{k}_{\ast}^{\top} & k(x_{\ast}, x_{\ast})
\end{bmatrix}
\right),
$$
where $[\mathbf{k}_{\ast}]_i = k(x^{(i)}, x_{\ast})$.
The mean and variance of the predictions can then be obtained by conditioning on the training data as
\begin{align*}
    \mu_{x_{\ast}} &= \mathbf{k}_{\ast}^{\top}(\mathbf{K} +  \sigma_n^2\mathbf{I})^{-1}\mathbf{y} \\
    \Sigma_{x_{\ast}} &= k(x_{\ast},x_{\ast}) -  \mathbf{k}_{\ast}^{\top}(\mathbf{K} +  \sigma_n^2\mathbf{I})^{-1}\mathbf{k}_{\ast}.
\end{align*} 
Note that $\theta = [\sigma_s \; \sigma_n \; W]$ are hyperparameters which are typically optimized by maximizing the log-likelihood of the training data (e.g. \cite{wang2019exact}) . The main drawback of GPR is the need to compute $(\mathbf{K} +  \sigma_n^2\mathbf{I})^{-1}$, which is typically performed using Cholesky decomposition resulting in a computational complexity of $\mathcal{O}(N^3)$, where $N$ is the number of training data points. To remedy this, many online GPR approximation algorithms have been proposed in the literature that leverage the sparsity of the covariance matrix. For example, in \cite{nguyen2009model}, the authors propose an online approximation of GPR that clusters the training data into multiple local GPR models. To handle datastreams, rank-one updates are performed to the Cholesky decomposition. Lastly, the prediction for a query point is then performed through a weighted average of the predictions of the local models. We refer the reader to the taxonomy in \cite{liu2020gaussian} for details on  scalable GPs.}

\subsection{GPR for Disturbance Estimation}
\label{subsec:gpr_for_estimation}

\todo{In this subsection, we discuss how GPR can be used to estimate the disturbance sets $D_A(x)$ and $D_M(x)$ introduced in Section~\ref{sec:barrier-functions-for-disturbed}. Note that, in this paper, we assume that the disturbance is not time-varying and leave this case for future work.

\subsubsection{Additive Disturbance Estimation}
Recalling the additive disturbance model from \eqref{eq:control-affine-disturbed}, we aim to estimate the disturbance set $D_A(x(t))$ through the use of GPR. The latter can be achieved by obtaining the dataset $\mathcal{D} = \{x^{(i)}, y^{(i)}\}^{N}_{i=1}$ with labels $y^{(i)}$ given by
\begin{equation}
    y^{(i)} = \hat{\dot{x}}^{(i)} - f(x^{(i)}) - g(x^{(i)})u^{(i)}, 
\end{equation}
where $\hat{\dot{x}}^{(i)}$ is the noisy measurement of the dynamics. Note that, in this case, $y^{(i)} \in \mathbb{R}^n$, therefore, we train one GP per dimension for a total of $n$ GP models.

Then, we can obtain the disturbance estimate for a query point $x_{\ast}$ as 
\begin{equation}
\label{eq:disturb_estimate}
[D_A(x_{\ast})]_i = \mu_i(x_{\ast}) + [-k_c \sigma_i(x_{\ast}), k_c \sigma_i(x_{\ast})],
\end{equation}
where $\mu_i(x_{\ast})$ and $\sigma_i(x_{\ast})$ are the mean and standard-deviation predictions of the $i$th GP for query point $x_{\ast}$.}


\vspace{2mm}
\todo{
\subsubsection{Multiplicative Disturbance}: Given the model 
\begin{gather*} 
    \dot{x}(t) \in f(x(t)) + (g(x(t) + D_M(x(t)))u(x(t)),
\end{gather*}
similarly to the additive case, we aim to obtain an estimate of the disturbance $D_M$ using GPR. Note that as opposed to $D_A$, $D_M$ is a matrix in $\mathbb{R}^{n \times m}$. The latter poses a difficulty in building the dataset $\mathcal{D}$ since each data point results in only $n$ equations, namely
\begin{gather*} 
    \hat{\dot{x}}^{(i)} - f(x^{(i)}) - g(x^{(i)})u^{(i)} =   y^{(i)}u^{(i)}, 
\end{gather*}
whereas $y^{(i)} \in \mathbb{R}^{n \times m}$. As such, we circumvent this by assuming that only $n$ user-chosen entries are non-zero, which is a reasonable assumption as we will show in the experiments.

For example, consider our application of interest: the fleet of differential drive robots operating in the Robotarium \cite{pickem2017robotarium}, a remotely accessible multi-robot testbed. Each robot has state $x \in \mathbb{R}^3$ composed of its global position in the plane and heading
$ x \triangleq
    \begin{bmatrix}
     x_{1} ~ x_{2} ~ \theta
     \end{bmatrix}^{\top}$ and disturbed dynamics 
$
    \dot{x} \in (g_x(x) + D_M(x))
    u(x) ,
$
where 
\begin{equation}
\label{eq:grits_gx}
g_x(x) \triangleq \begin{bmatrix}
        \cos \theta  & 0 \\ 
        \sin \theta  & 0 \\ 
        0 & 1 
    \end{bmatrix}, 
\end{equation}
\begin{equation*}
    D_M(x) = \begin{bmatrix}
    \co \psi_{11}(x) & \co \psi_{12}(x) \\
    \co \psi_{21}(x) & \co \psi_{22}(x) \\
    \co \psi_{31}(x) & \co \psi_{32}(x)
    \end{bmatrix},
\end{equation*}
and the unicycle model input 
$u \triangleq
    \begin{bmatrix}
         v~ \omega
    \end{bmatrix}^{\top}$, with $v$ and $\omega$ denoting the linear and angular velocities respectively. We assume that the commanded angular velocity can have no effect on the linear velocity of the robot in the plane, thus obtaining that $\co \psi_{12}(x) = \co \psi_{22}(x) = {0}$. Similarly, we also assume that the linear velocity can have no effect on $\dot{\theta}$ which results in $\co \psi_{31}(x) = 0$, yielding the following disturbance structure
    \begin{equation}
    \label{eq:DM_grits}
    D_M(x) = \begin{bmatrix}
    \co \psi_{11}(x) & \{0, 0\} \\
    \co \psi_{21}(x) & \{0, 0\} \\
    \{0, 0\}              & \co \psi_{32}(x)
    \end{bmatrix}.
    \end{equation}
As a result, we are now able to estimate each of the $n=3$ non-zero entries of $D_M(x)$ using a separate GP model, for a total of $n$ models.}  


\begin{figure}[t]
    \centering
    \includegraphics[width=\linewidth]{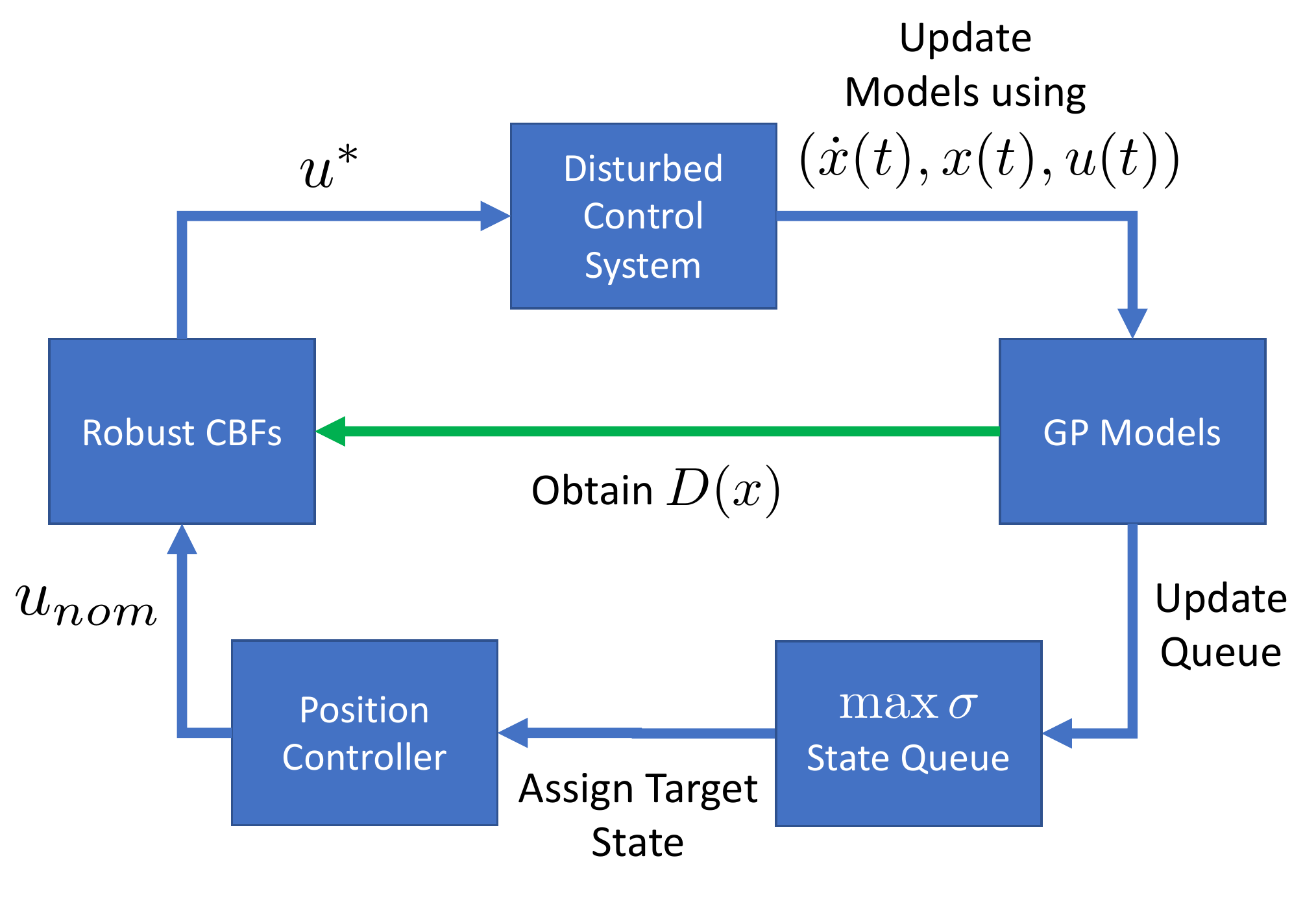}
    \caption{Schematic of the proposed controller-synthesis framework using GPs for disturbance estimation. The robots efficiently explore the environment by visiting the points with highest variance ($\text{argmax}_x \sigma(x)$ state queue). The points collected are used to estimate the disturbance and update the queue using GPs. These estimates are then utilized by the Robust CBFs to guarantee safety. Note that the GP models block can be replaced by any other data-driven disturbance-estimation method.} 
    \label{fig:control-synth-schematic}
\end{figure}

\section{Controller Synthesis}
\label{sec:control-synthesis-via}

In this section, we introduce a QP-based controller-synthesis framework for disturbed control-affine systems that combines the results from Sections~\ref{sec:barrier-functions-for-disturbed}~and~\ref{sec:disturb-estimation}. The framework estimates the disturbance online using GPR as described in the previous section.  The choice of modeling control-affine system is typically nonrestrictive since many control systems (especially mechanical) are control affine.  We choose to focus on the multiplicative disturbance case and note that the framework can be straightforwardly altered for the additive disturbance case and other disturbance estimation methods.  

Consider a disturbed control affine dynamical system modeled by \eqref{eq:control-affine-disturbed-mult}, where $D_M(x(t))$ is estimated using GPR as described in the previous section. Given a nominal control $u_{\text{nom}}$, we would like to ensure the system's safety by minimally altering $u_{\text{nom}}$. This can be achieved through a QP as shown in the next proposition. 

\begin{proposition}
    \label{prop:control-synthesis}
   Let $h : \mathbb{R}^{n} \to \mathbb{R}$ be a continuously differentiable function and $\Phi(x') = \{\phi_i(x')\}^{2^m}_{i=1}$ be the set of vertices of the $m$-orthotope defined by 
$$\nabla h(x')^{\top} D_{M}(x'),$$
where $D_{M}(x')$ is defined as in \eqref{eq:d_est-m}.  If $u^{*} : \mathbb{R}^{n} \to \mathbb{R}^{m}$ defined as  
    \begin{align}
        \label{QP}
        & u^{*}(x') = \argmin_{u \in \mathbb{R}^{m}} \|u_{\text{nom}}(x') - u\|^2 \\
            & \nabla h(x')^{\top}(f(x') + g(x') u) + \phi_i(x') u\geq \\
            & -\alpha(h(x')), \forall x' \in \mathbb{R}^{n} , \forall i \in \{1,\ldots, 2^m\},
    \end{align}
    is continuous, then $h$ is valid robust CBF for {\eqref{eq:control-affine-disturbed-mult}}.
\end{proposition}

\todo{Note that the existence of $u^\ast$ is assumed and that its continuity may be guaranteed under certain additional assumptions that are very similar to the ones presented in Theorem~$3$ from \cite{7782377}, which states sufficient conditions for the continuity of QP control in the case of regular CBFs. Specifically, if $u_{\text{nom}}$, $f$, $g$, $\phi_i, \; \forall i$ and  $\nabla h$ are all locally Lipschitz continuous, and $\nabla h(x')^{\top}g(x') \neq 0$ for all $x' \in \text{Int}(\mathcal{C})$, then $u^*$ is locally Lipschitz continuous for $x' \in \text{Int}(\mathcal{C})$. We refer the reader to \cite{morris2013} for further details.} This QP assumes that the nominal control  $u_{\text{nom}}$ is given, and minimally (in the least square sense) alters it such that $h$ remains a CBF. Since this safety framework is control agnostic, as described later in Section~\ref{sec:experiments}, we choose to implement a position controller  to compute $u_{\text{nom}}$ which efficiently samples the environment by visiting the state with the highest variance (i.e., $x_\text{next} = {\text{argmax}}_{x \in \mathcal{X}} (\sigma(x))$,  where $\mathcal{X} \subset \mathbb{R}^{n}$ is a discretization of the state space. As depicted in Figure~\ref{fig:control-synth-schematic}, the updated GP models are in turn used to obtain $D_M(x)$ thus creating a feedback loop between the disturbance estimation method, the robust CBFs and the disturbed dynamical system.

\section{Specialization of results for Robotarium}  \label{sec:robust-collision-avoidance}

In this section, we discuss the application of the robust CBFs detailed in Section~\ref{sec:barrier-functions-for-disturbed} and the controller-synthesis framework in Section~\ref{sec:control-synthesis-via} on a fleet of differential-drive robots in the Robotarium. The platform automatically and continuously runs experiments  of which many do not incorporate collision avoidance (e.g., \cite{Hatanaka2009,Park2016}). To handle this, a minimally invasive CBF QP framework is provided to the users \cite{pickem2017robotarium}. However, the fact that robots are susceptible to issues (e.g., wheel slip) coupled with the required long-duration autonomy of the system motivate the need for a robust version of the collision-avoidance framework and make the Robotarium a fitting choice for this application. 

We assume that each robot obeys unicycle dynamics and that the disturbance is multiplicative with respect to the input as in subsection~\ref{subsec:cbf_mult}. The choice of modelling the disturbance as multiplicative was motivated by data collected on the Robotarium indicating a correlation between the input and the magnitude of the disturbance. Moreover, as in Section~\ref{sec:disturb-estimation}, GPs are leveraged to estimate the disturbance. For the sake of clarity, we omit the explicit dependence on time for brevity. 

Consider $N$ differential-drive robots, where each robot has state $x_{i} \in \mathbb{R}^3$ composed of its global position in the plane and heading
\begin{gather} 
    \label{eq:robDynamics}
    x_{i} \triangleq
    \begin{bmatrix}
     x_{i, 1} ~ x_{i, 2} ~ \theta_{i}
     \end{bmatrix}^{\top} .
\end{gather}
Each robot has dynamics
\begin{gather}
\label{eq:gritsbot-dynamics-disturbed}
    \dot{x}_{i} \in (g_x(x_i) + D_M(x_i))
    u_{i}(x_{i}) ,
\end{gather} 
where $g_x$ and $D_M$ are defined as in \eqref{eq:grits_gx} and \eqref{eq:DM_grits} respectively, and the unicycle model input as
$
    u_{i} \triangleq
    \begin{bmatrix}
         v_{i},~ \omega_{i} 
    \end{bmatrix}^{\top} ,
$ where $v_{i}$ and $\omega_{i}$ are the linear and angular velocities of robot $i$, respectively. This geometric model is representative of the control methodology of the Robotarium's robots shown in Figure~\ref{fig:GRITSBOT}.

Since the GRITSBot-Xs' wheels are located toward the back, as shown in Figure~\ref{fig:GRITSBOT}, we model the location of the centroid of each robot as an output of its state 
\begin{align}
    p_i(x_{i}) \triangleq 
    \begin{bmatrix}
        x_{i1} \\
        x_{i2}
    \end{bmatrix}
    + l_{p}
    \begin{bmatrix}
        \cos{\theta_i} \\
        \sin{\theta_i} 
    \end{bmatrix} .
\end{align}
where $l_p > 0$ is the distance between the wheels' axis and the centroid of the robot. This permits the formulation of a collision-avoidance constraint from the centroid of the robot. 

In order to obtain $\dot{p}_i(x_i)$, we need to propagate the disturbance through the dynamics of the output $p_i(x_i)$. Since we model the disturbance using convex-hulls, this process is straightforward and is another advantage of our approach. Differentiating $p_i$ along the unicycle dynamics yields
\begin{gather} 
    \label{eq:pi_dot}
    \dot{p}_i(x_{i}) = G_p(x_{i})u_{i}(x_{i}) ,
\end{gather}
where
\begin{align}
    G_{p}(x_{i}) &\triangleq R(\theta_{i})L(x_i) + 
    \begin{bmatrix}
    1 & 0 & 0 \\ 
    0 & 1 & 0
    \end{bmatrix}
    D_M(x_i),  \\\\
    R(\theta_{i}) &\triangleq 
        \begin{bmatrix}
            \cos \theta_{i} & -\sin \theta_{i} \\
            \sin \theta_{i} & \cos \theta_{i} 
        \end{bmatrix},\\\\
    L(x_i) &\triangleq 
    \begin{bmatrix}
        1 & 0 \\ 
        l_{p}\co \psi_{31}(x_i) & l_{p} (1+\co \psi_{32}(x_i))
    \end{bmatrix}.
\end{align}
\vspace{.5mm}

\noindent
For later convenience, define the ensemble variables
\begin{gather}
    u \triangleq
    \begin{bmatrix}
        u^{\top}_1 & \ldots & u_N^{\top}
    \end{bmatrix}^{\top},~ %
    x \triangleq 
    \begin{bmatrix}
    x^{\top}_1 & \ldots & x_N^{\top}
    \end{bmatrix}^{\top} \\
   p \triangleq 
    \begin{bmatrix}
        p^{\top}_1 & \ldots & p_N^{\top}
    \end{bmatrix}^{\top}  
\end{gather}

Using the fact that $l_{p}$ is chosen so that $p_{i}$ is at the centroid of each robot, the following CBF encodes a collision-avoidance constraint between robots $i$ and $j$
\begin{equation} 
    \label{eq:collision-avoidance-cbf}
    h_{ij}(x) \triangleq \norm{p_{i}(x_{i}) - p_{j}(x_{j})}^2 - \delta^{2} ,
\end{equation}
where $\delta > 0$ denotes the diameter of the robot.  Note that
\begin{equation}
    \nabla_{p(x_{i})} h_{ij}(x) = (p_{i}(x_{i}) - p_{j}(x_{i})) = - \nabla_{p(x_{j})} h_{ij}(x). 
\end{equation}

The CBF in \eqref{eq:collision-avoidance-cbf} models each robot as a circle of diameter $\delta$. Again, note that the center of each robot's wheel axle is shifted from the center of the circle.  Conveniently, this issue can be easily mitigated by setting the look-ahead distance $l_p$ to map the point $p_i$ to the center of the robot. 

The barrier certificate that needs to be satisfied for each pair $(i, j)$ of robots is
\begin{align}
    \label{eq:barrier-certificate}
    &\min(\nabla_{p_i(x_i)} h_{ij}(x)^{\top}G_p(x_i)u_{i} \\
    & \hspace{4mm} + \nabla_{p_j(x_j)} h_{ij}(x)^{\top}G_p(x_j)u_{j}) \geq - \gamma h_{ij}(x)^3, 
\end{align}
where the chosen extended class-$\mathcal{K}$ function has been $\alpha(s) = \gamma s^3$ for some $\gamma > 0$. 
We begin by, as described in Section~\ref{sec:disturb-estimation}, splitting \eqref{eq:barrier-certificate} into $2^{2m}$ constraints (16 in this case) in order to obtain linearity with respect to $u$. Note that \eqref{eq:barrier-certificate} is equivalent to
\begin{align}
\min(Q^{(i,j)}
\begin{bmatrix} u_i \\ u_j 
\end{bmatrix}) \geq - \gamma h_{ij}(x)^3, 
\end{align}
where $$Q^{(i,j)} \triangleq [\nabla_{p_i(x_i)} h_{ij}(x)^{\top}G_p(x_i), \nabla_{p_j(x_j)} h_{ij}(x)^{\top}G_p(x_j)].$$
Note that,  since $\nabla_{p_i(x_i)} h_{ij}(x)^{\top}G_p(x_i)$ is a vector of size $2$,  $Q^{(i,j)}$ has $4$ entries each of which is an interval (i.e., a convex-hull of 2 points), hence $Q^{(i,j)}$ can be equivalently represented as the convex hull of $16$ vectors
\begin{equation} 
Q^{(i,j)} = \co(\{q^{(i,j)}_1,\ldots, q^{(i,j)}_{16}\}) \textrm{ where } q^{(i,j)}_k \in \mathbb{R}^{4 \times 1}.
\end{equation}
By applying Proposition~\ref{prop:valid-cbf-disturbed-m2} we obtain the synonymous representation of \eqref{eq:barrier-certificate} as
$$q_k^{(i,j)}
\begin{bmatrix} u_i \\ u_j 
\end{bmatrix} \geq - \gamma h_{ij}(x)^3 \;\; \forall k \in \{1,\ldots,16\}.$$

\begin{figure}[t]
    \centering
    \includegraphics[width=0.30\textwidth]{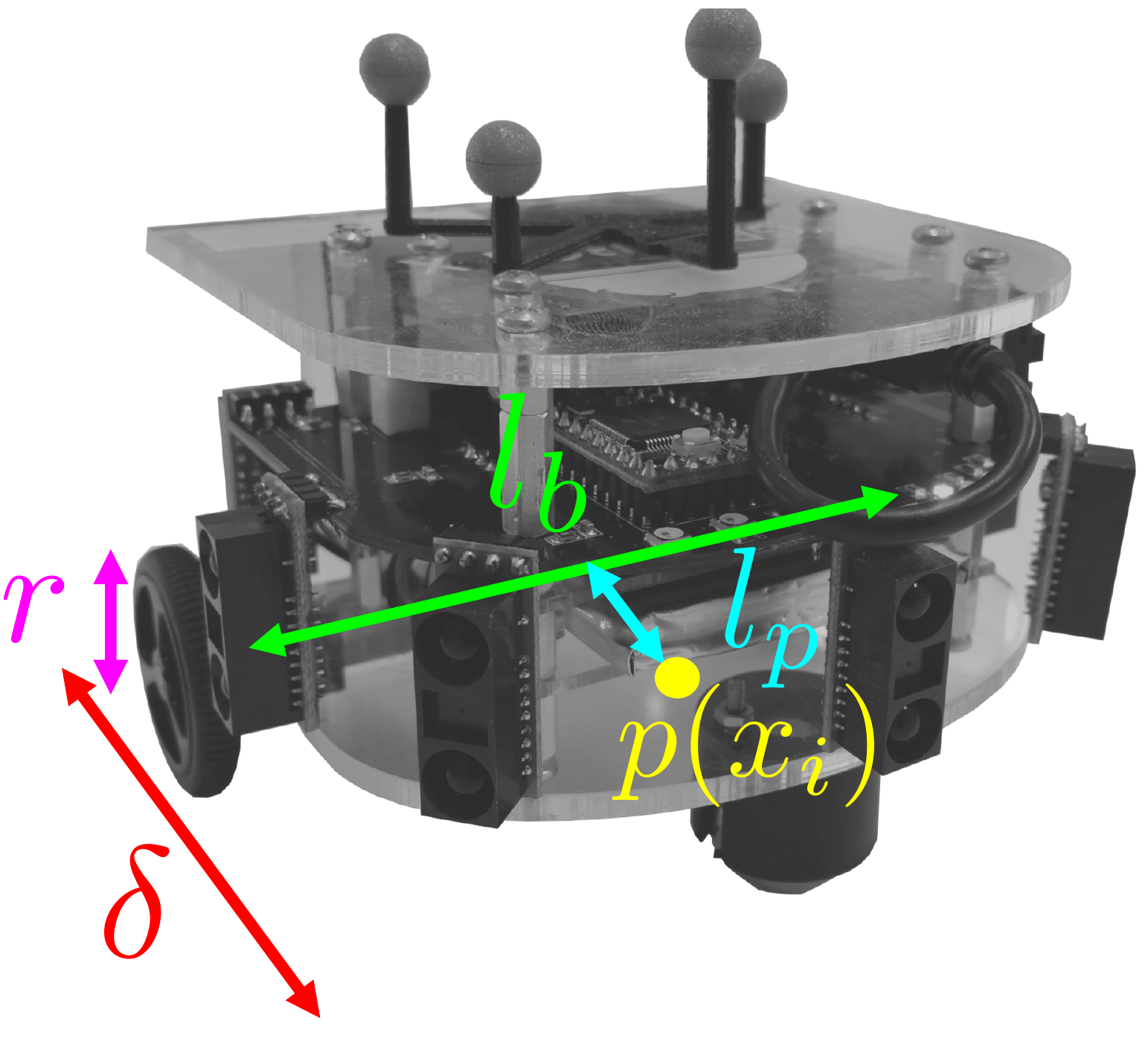}
    \caption{A picture of the GRITSbot, the differential-drive robot used by the Robotarium. The symbol $l_b$ denotes the base (i.e. wheel axle) length, $r$ the wheel radius, $l_p$ the projection distance, $\delta$ the diameter of the GRITSbot, and $p$ the center of the robot.  Note that $l_p$ projects the wheel-axle's center-line, denoted by the green-line, to the center-line of the robot.}
    \label{fig:GRITSBOT}
\end{figure}

Now, it remains to formulate \eqref{eq:barrier-certificate} as in a form conducive to an ensemble optimization program as in \cite{pickem2017robotarium}.  It is then convenient to define a matrix-valued function $A$ and a vector-valued function $b$ so that the constraint can take the ensemble form $A(x)u \geq b(x)$. We do so by defining the following matrix $A_{(i,j)}(x) \in \mathbb{R}^{16 \times 2N}$ and vector $b_{(i,j)}(x)$ for each pair $(i,j)$ of robots
\begin{align}
    & A_{(i,j)}(x) = 
    \begin{bmatrix}
        \bar{q}_1^{(i,j)\top}& \ldots & \bar{q}_{16}^{(i,j)\top},
    \end{bmatrix}^\top \\
    & b_{(i,j)}(x) = - \gamma h_{ij}^3(x) \mathbf{1}_{16},
\end{align}
where $\bar{q}^{(i,j)}_k \in \mathbb{R}^{2N}$ is the ensemble form of of  ${q}^{(i,j)}_k$. Note that this process is repeated $\forall i \in \{1,\ldots,N-1\}$ and $\forall j \in \{i+1,\ldots,N\}$ to account for all possible pairwise-collisions, and the $A$ and $b$ matrices are then obtained by vertically stacking all the generated $A_{(i,j)}$ and $b_{(i,j)}$, respectively.  At last, by using all the variables introduced so far, we can formulate a similar QP to the one in Proposition~\ref{prop:control-synthesis} as
\begin{align}
\label{eq:QP2}
    u^*=\underset{u \in \mathbb{R}^m}{\text{argmin.}} \quad & ||L_c(u_{\text{nom}}(x) - u)||^2 \\
    \text{s.t.} \quad  & \left \| u \right \|_{\infty}  \leq u_{\text{max}}\\
    & A(x)u \geq b(x) , 
\end{align}
where $$L_c =  I_{N} \otimes 
\begin{bmatrix}
1 & 0 \\
0 & l_p
\end{bmatrix}$$ 
is a weighting matrix that alleviates dead-lock situations by encouraging alterations in the angular velocities rather than the linear velocities and $\otimes$ denotes the Kronecker product.  This formulation ensures that the barrier functions minimally alter $u_{\text{nom}}$ to render the input safe while taking actuation limits into account. \todo{To account for actuator limits, \eqref{eq:QP2} includes the constraint $\|u\| \leq u_{\text{max}}$, which does not change the results of Proposition~\ref{prop:control-synthesis} where the existence of a solution is assumed. Note, however, that in this specific scenario, a solution to \eqref{eq:QP2} is always guaranteed to exist since the robots are modelled as single-integrators ($u=0$ can always satisfy the constraints). For more general cases, we refer the reader to \cite{ames2019control, wang2017} and \cite{squires2018constructive} for approaches on how to construct CBFs that guarantee the existence of solutions under actuator constraints which typically rely on a backup safe maneuver.} In the next section, we present experimentation demonstrating how the robust CBF formulation decreases the number of constraint violations during autonomous operation of the Robotarium.

\section{Experiments}
\label{sec:experiments}

We present two experimental scenarios to demonstrate the efficacy of the proposed framework on a team of differential drive robots. The various experimental parameters (e.g., dimensions of the robots pictured in Figure~{\ref{fig:GRITSBOT}}) are presented in Table~{\ref{tab:parameters}}. Moreover, the robust CBF formulation used in the experiments utilizes the multiplicative-disturbance case as explained in Section~\ref{sec:barrier-functions-for-disturbed}.

\begin{table}[tb]
    \centering
    \caption{Values of the relevant GRITSBot-X's and experiments' parameters.}
    \begin{tabular}{|c|c|c|c|c|c|}
        \hline
       $l_p$ (m)  & $l_b$ (m) & $r$ (m) & $\delta$ (m) & $\gamma$ & $k_c$ \\
       \hline
        $0.03$& $0.105$ & $0.016$& $0.12$& $700$ &  $2$ \\
       \hline
    \end{tabular}
    \label{tab:parameters}
\end{table}

\subsection{Online Learning of Spatially Varying Disturbance}
\label{subsec:exp-spatially-varying}

The first experiment demonstrates the application of the controller-synthesis procedure presented in Section~{\ref{sec:control-synthesis-via}} in an online safe-learning scenario.  In this case, the word safe refers to collision avoidance.  In particular, GP models are used, in real time, to estimate the disturbance.  In turn, the robust-CBF framework, as described by the specialized results in Section~\ref{sec:robust-collision-avoidance}, utilizes the GPs to ensure the safety of the robots. To learn the underlying disturbance, the robots are tasked to drive to states (obtained through a discretization of the state space) for which the variance is highest.  This variance-based strategy is commonly used to learn the environmental disturbance in a sample-efficient fashion. 

\todo{As discussed in subsection \ref{subsec:gpr_for_estimation}, motivated by the physical constraints of the GRITSBot-X, we assume that the disturbance $D_M$ used in the dynamics \eqref{eq:gritsbot-dynamics-disturbed} is given by \eqref{eq:DM_grits}, where each nonzero $\psi_{ij}(x) = \{\mu_{ij}(x) - k_c \sigma_{ij}(x), \mu_{ij}(x) + k_c \sigma_{ij}(x)\}$ is estimated using its own GP model.} As such, for a data point $k$, the labels for the three GP models are generated using the following equations
\begin{align*}
    [y^{(k)}]_{11} &= \hat{\dot{x}}^{(k)} / v^{(k)} - \cos{(\theta^{(k)})},\\
    [y^{(k)}]_{21} &= \hat{\dot{y}}^{(k)} / v^{(k)} - \sin{(\theta^{(k)})},\\
    [y^{(k)}]_{32} &= \hat{\dot{\theta}}^{(k)} / \omega^{(k)} - 1,
\end{align*}
where $\hat{\dot{x}}^{(k)}$, $\hat{\dot{y}}^{(k)}$ and $\hat{\dot{\theta}}^{(k)}$ denote the measured velocities of the states in \eqref{eq:robDynamics} for data point $k$. Note that singularities can occur if $v^{(k)}$ or $\omega^{(k)}$ equal zero, therefore we discard such data points. \todo{Moreover, the disturbance models are updated asynchronously for every new $50$ data points collected. Initially, when no data points have been collected, we use a large disturbance interval $\psi_{ij}(x)~=~\{\psi^{\text{min}}_{ij}, \psi^{\text{max}}_{ij}\}$, until the initial $50$ data points have been collected. Then, at each time step $G_p$ and $L$ are computed using the most recent $D_M$.}

As the Robotarium testbed is relatively uniform, we expect a major component of the disturbance to stem from imperfections intrinsic to the robot---a component that is not spatially varying. Therefore, in order to demonstrate the ability of the proposed disturbance estimation framework of capturing spatially varying disturbances, in the top-left quarter of the arena we simulate a disturbance by multiplying the input of the robots by $0.80$.

Live screenshots of the experiment at the beginning, midpoint, and end are shown in Figure~\ref{fig:exp1-screenshots}.
Note that, in addition to the real disturbance inherent to the system, the simulated disturbance induced in the top left quarter of the arena is equivalent to multiplying the matrix $g_x(x_i)$ from \eqref{eq:gritsbot-dynamics-disturbed} by $0.8$, resulting in a non-trivial disturbance $D^{\text{sim}}_M(x_i) = -0.2g_x(x_i)$. Shown in Figures~\ref{fig:exp1-mu_corners_tl} and \ref{fig:exp1-mu_corners_br} is the estimate of the disturbance at the the top-left and bottom-right corners of the arena respectively. As shown in the plots, the true simulated disturbance for both corners (dashed red lines) falls between the $95\%$ confidence bounds (gray fill) of the estimate in all plots, indicating that our estimation does indeed capture the simulated disturbance. We note however, that as shown in the right hand plot of Figure~\ref{fig:exp1-mu_corners_tl}, the mean estimated disturbance on $\dot{\theta}$ does not match the expected effect of the simulated disturbance, which is most likely due to the effect of the disturbance inherent to the GRITSBot-Xs.   

Figure~\ref{fig:exp1-minh} shows a plot of $\min_{ij} h_{ij}$ over time, which showcases that zero constraint violations occur throughout the experiment.  This lack of violations demonstrates the efficacy of the proposed robust-CBF formulation with respect to safe learning under a spatially varying disturbance. Overall, in this experiment, the robots were able to learn the underlying disturbance using the GP-based approach while avoiding inter-robot collisions.

\begin{figure}[tb]
    \centering
    \includegraphics[width=\linewidth]{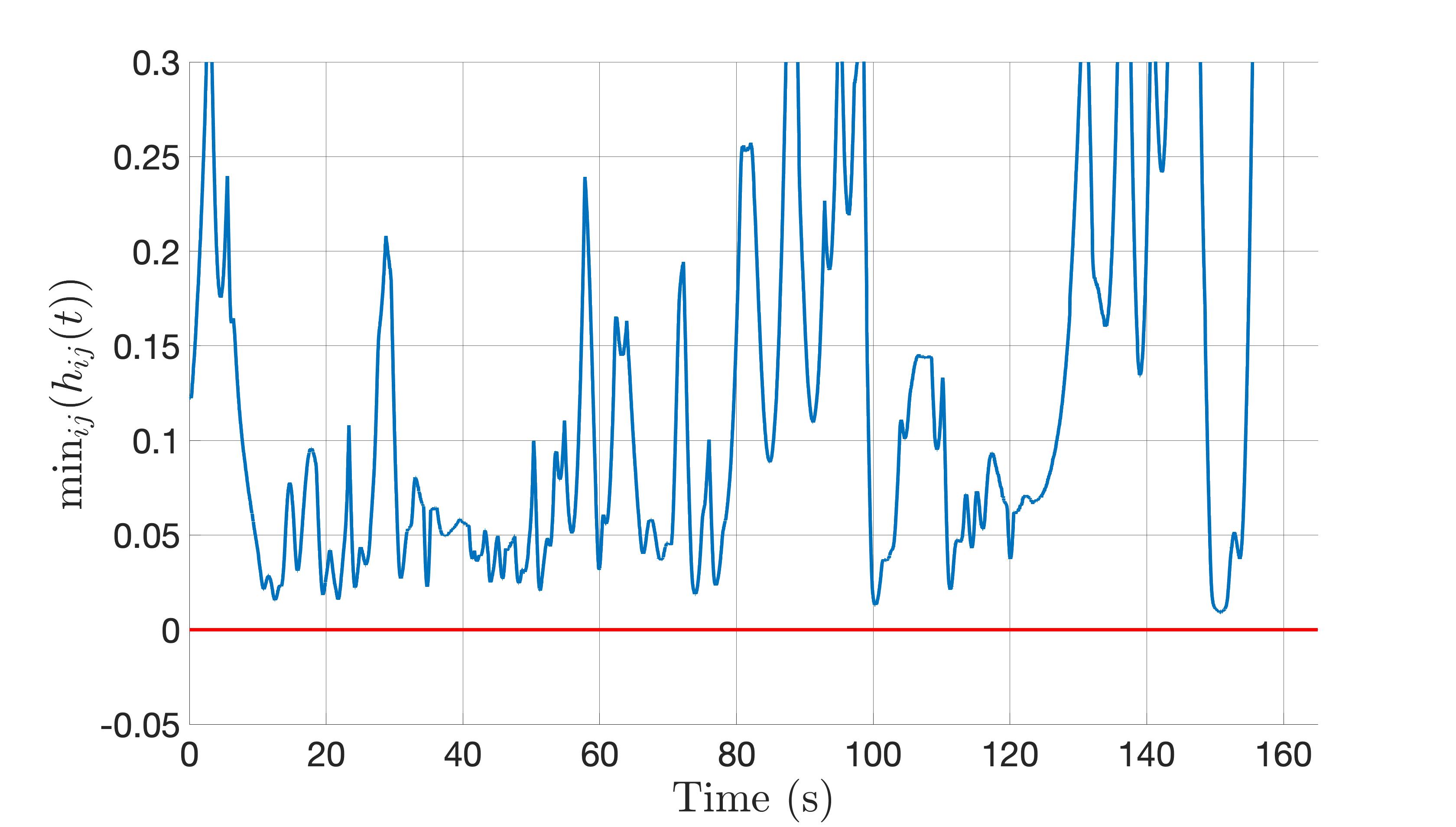}
    \caption{Plot of $\min_{ij}h_{ij}$ over time for experiment $1$, where a simulated spatially varying disturbance is applied to the robots. As highlighted by the plot, $0$ constraint violation occur with the robust CBF formulation demonstrating its ability to account for spatially varying disturbances.}
    \label{fig:exp1-minh}
\end{figure}
\begin{figure*}[tb]
    \centering
    \includegraphics[height=32mm]{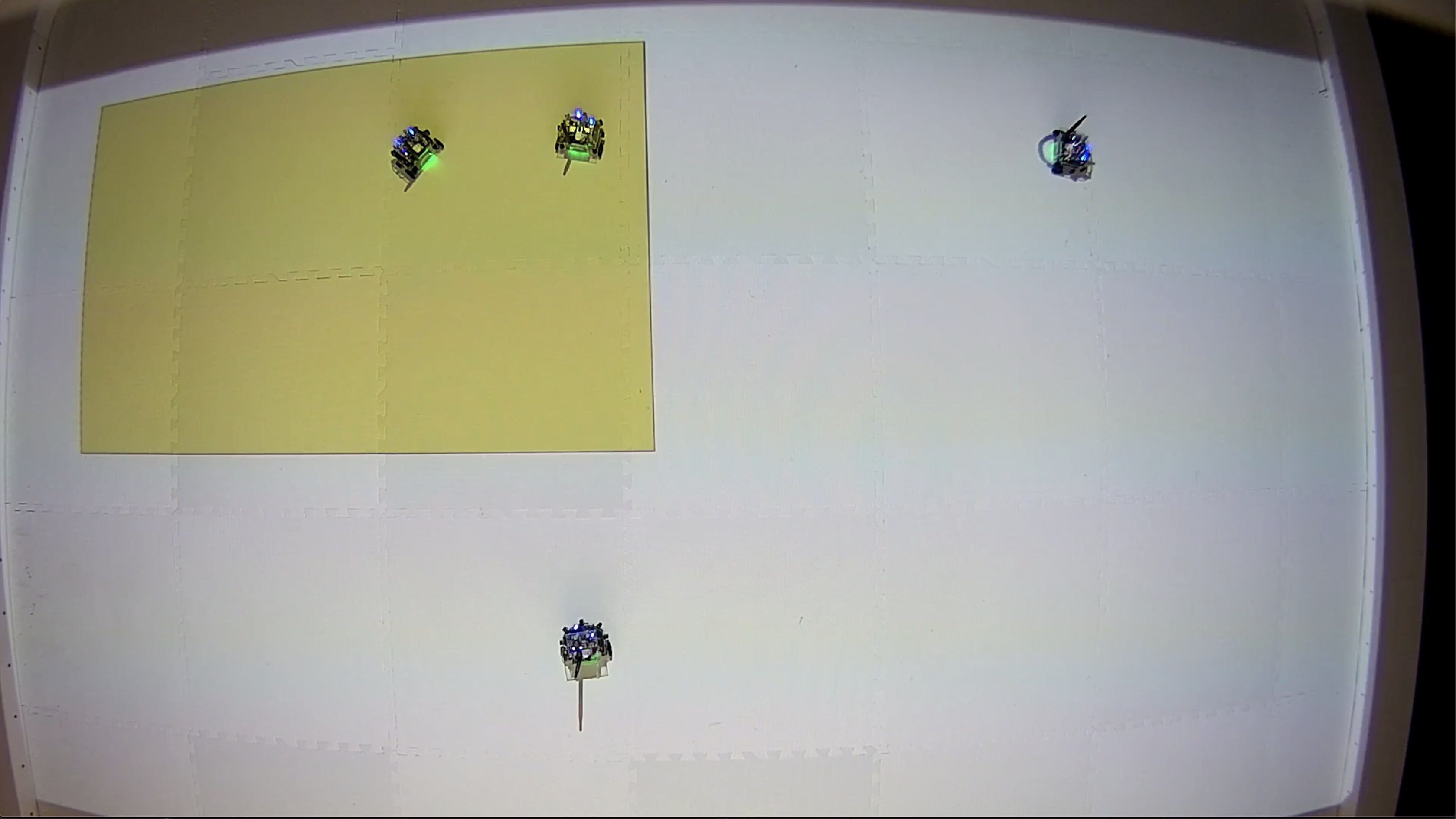}
    \includegraphics[height=32mm]{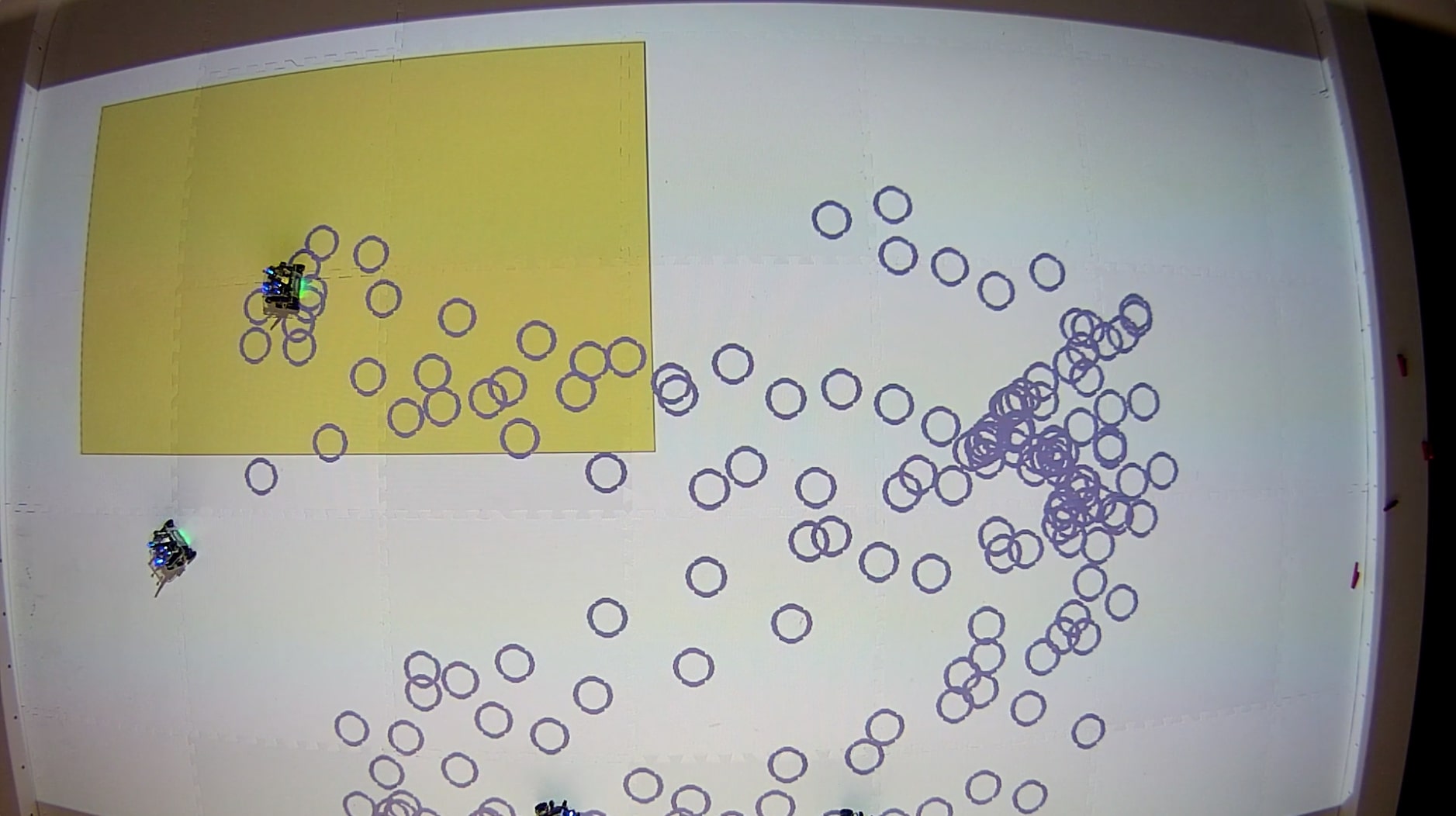}
    \includegraphics[height=32mm]{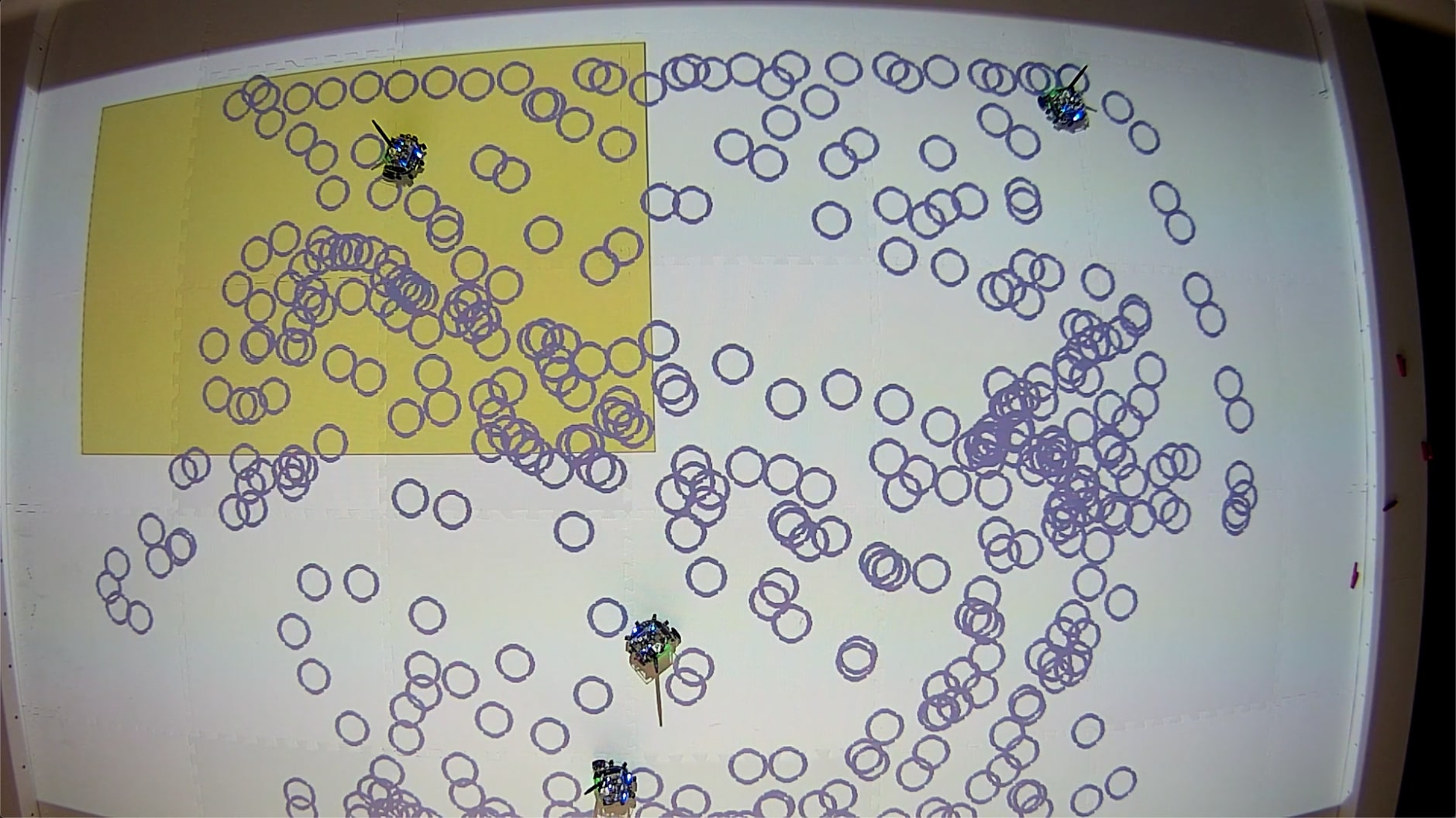}
    \caption{Live screenshots of experiment $3$'s beginning (left), halfway point (middle) and end (right). The yellow region denotes an area where a simulated disturbance---the input of the robots are multiplied by $0.8$---is applied. The blue circles denote points where data was collected for fitting the GP models aimed at estimating the disturbance.}
    \label{fig:exp1-screenshots}
    \vspace{4mm}
\end{figure*}
\begin{figure*}[tb]
    \centering
    \includegraphics[height=36mm]{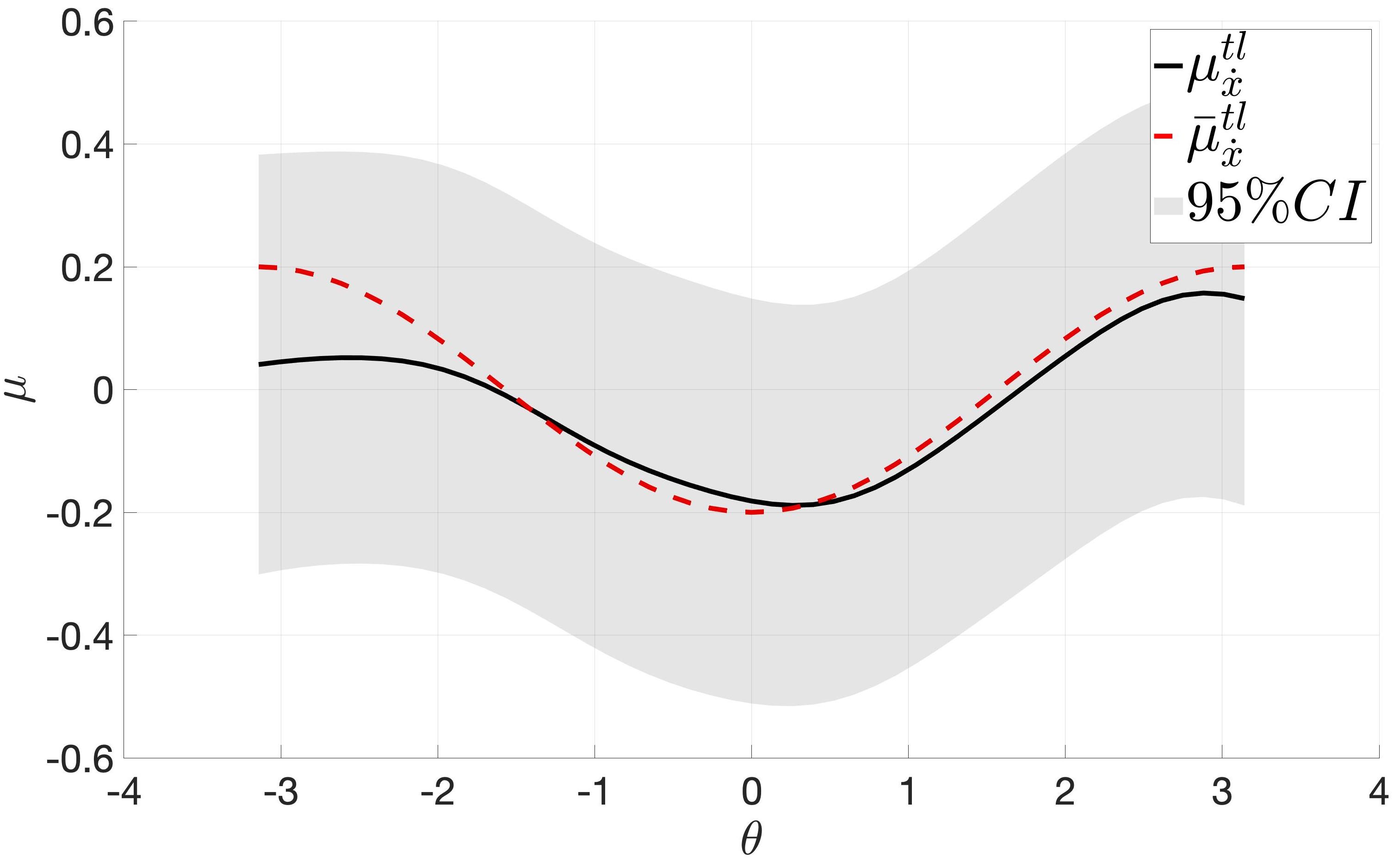}
    \includegraphics[height=36mm]{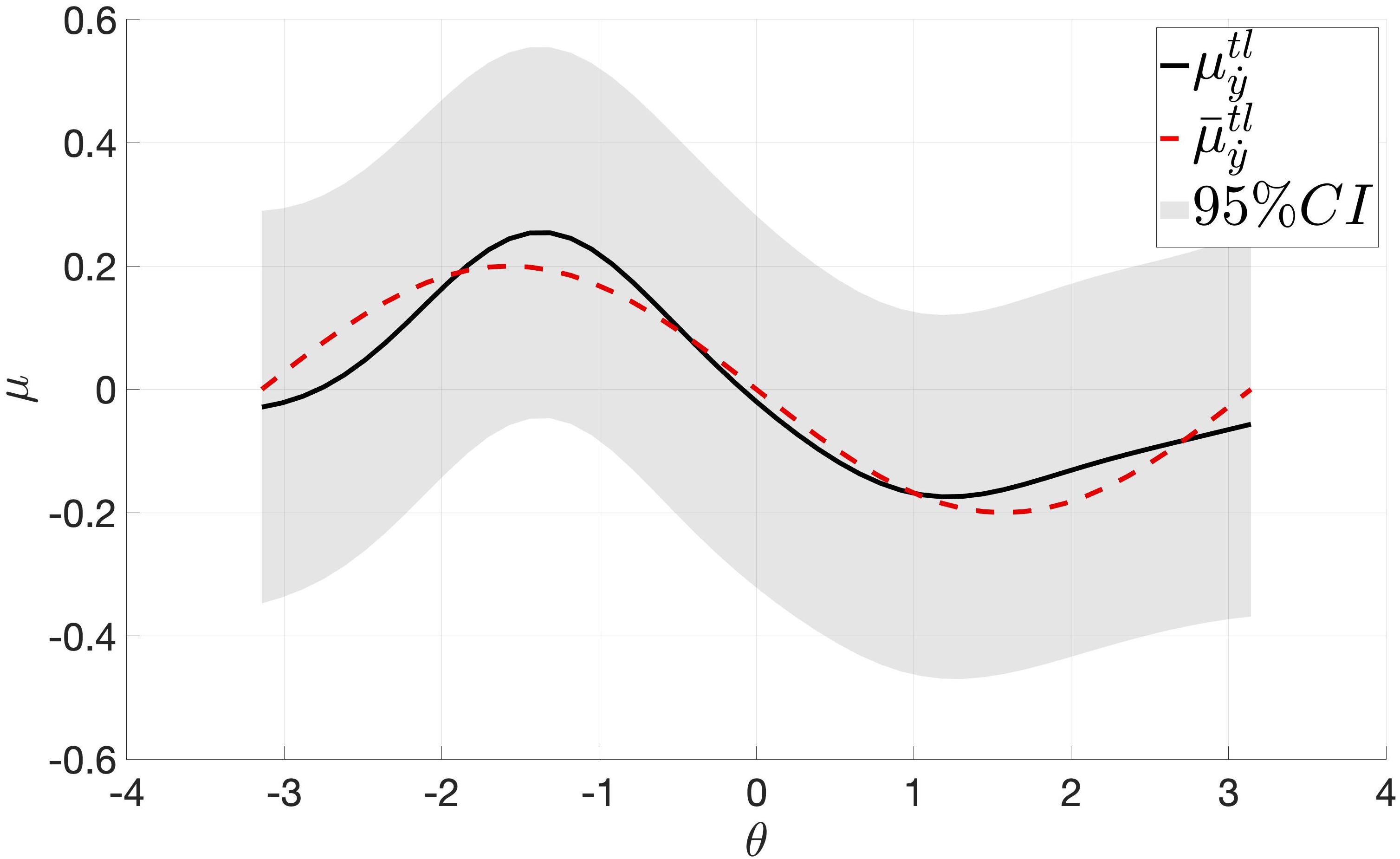}
    \includegraphics[height=36mm]{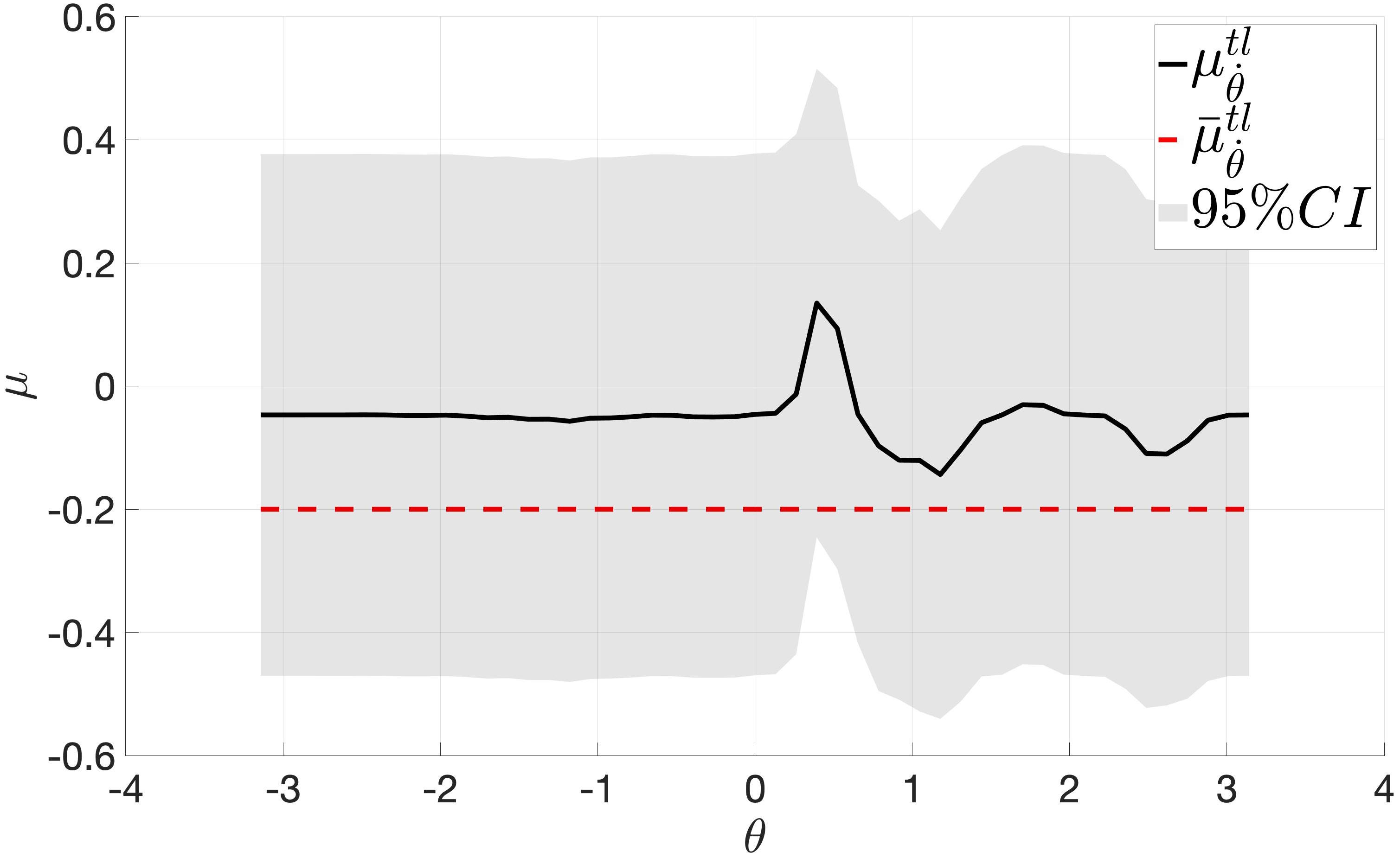}
    \caption{The disturbance estimate on $\dot{x}$ (left), $\dot{y}$ (middle) and $\dot{\theta}$ (right) at the top-left ($tl$) corner of the arena and the fill represents the $95\%$ confidence interval for experiment $1$. The dashed red line is the true simulated disturbance at the top left corner ($-0.2g(x)$) of the arena. It is clear the that the dashed red line is within the confidence bounds in each plot indicating that our estimate does indeed capture the simulated disturbance.}
    \label{fig:exp1-mu_corners_tl}
    \vspace{4mm}
\end{figure*}
\begin{figure*}[tb]
    \centering
    \includegraphics[height=36mm]{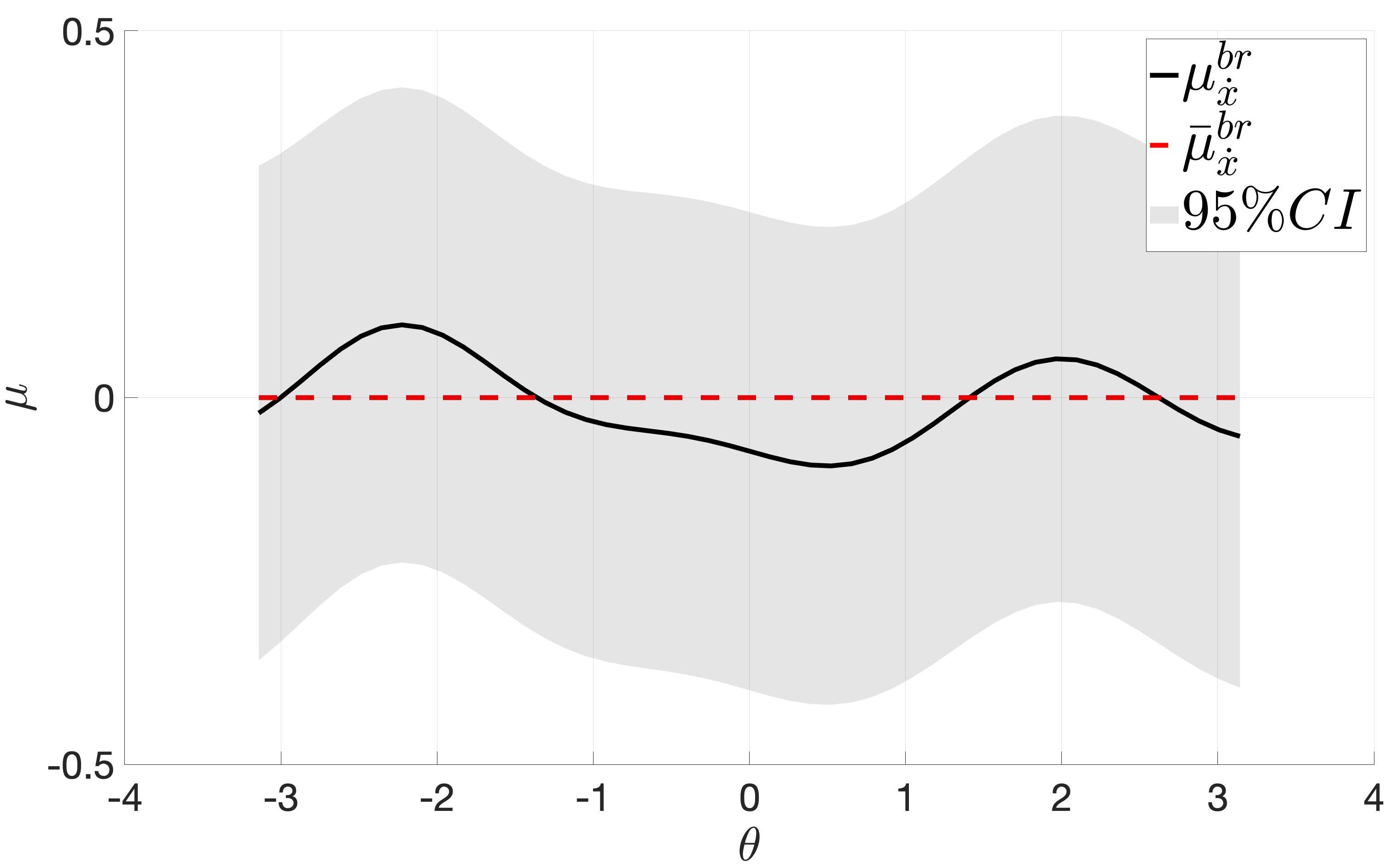}
    \includegraphics[height=36mm]{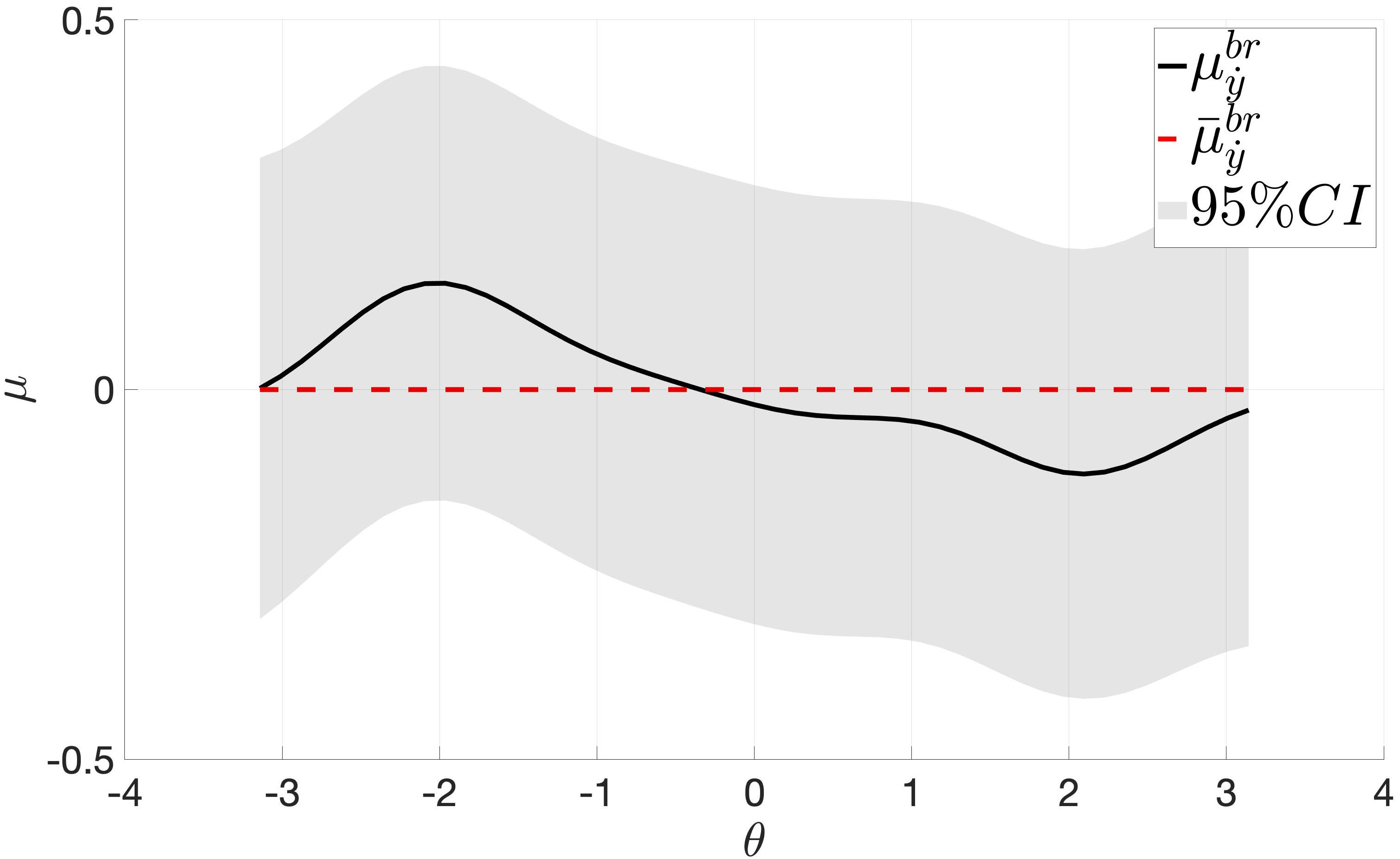}
    \includegraphics[height=36mm]{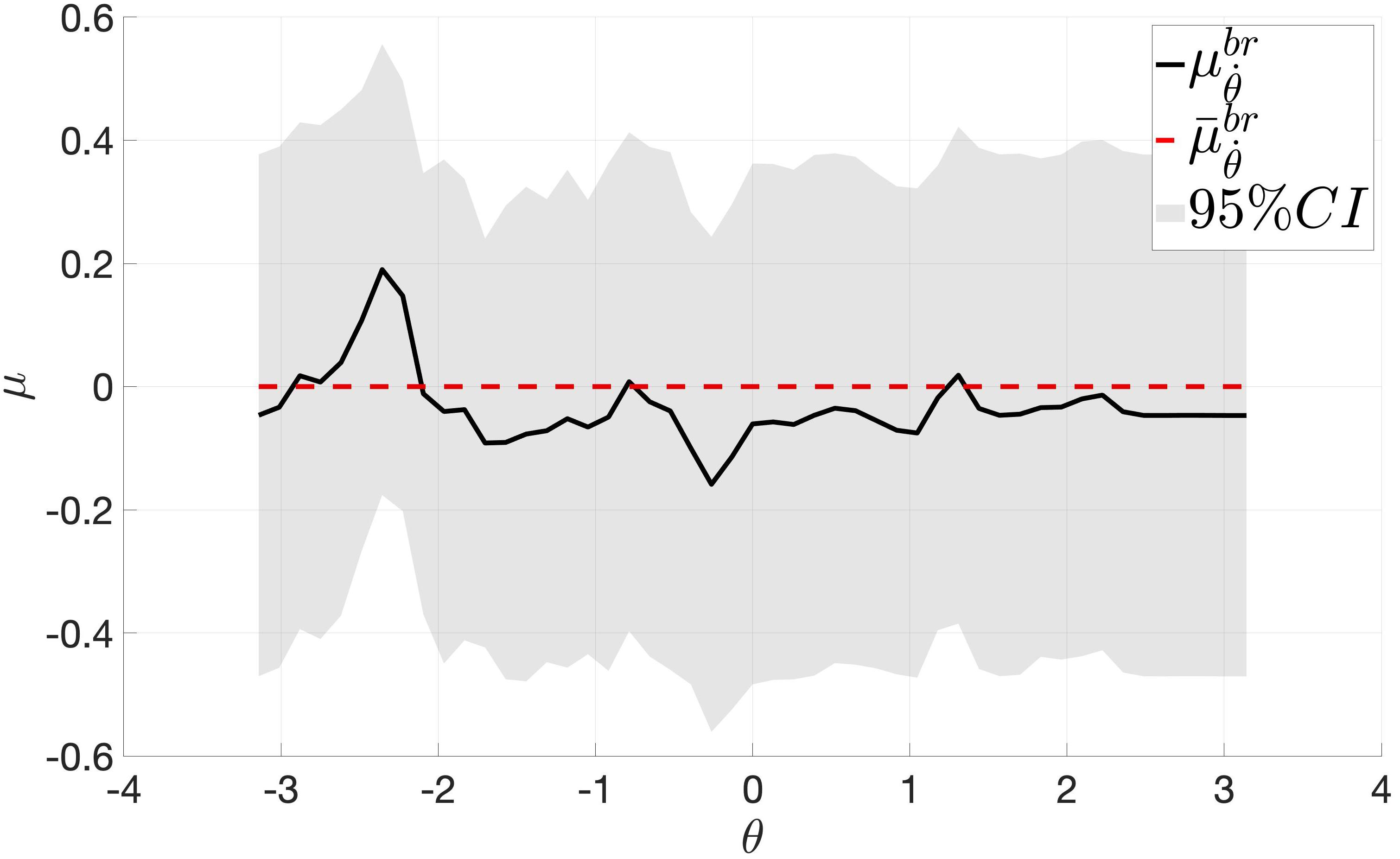}
    \caption{The disturbance estimate on $\dot{x}$ (left), $\dot{y}$ (middle) and $\dot{\theta}$ (right) at the bottom-right ($br$) corner of the arena.  The filled area represents the $95\%$ confidence interval for experiment $1$, and the dashed red line is the true simulated disturbance at the bottom-right corner ($0$) of the arena. It is clear the that the dashed red line is within the confidence bounds in each plot indicating that our estimate does indeed capture the simulated disturbance.}
    \label{fig:exp1-mu_corners_br}
    \vspace{4mm}
\end{figure*}
\begin{figure*}[tb]
\centering
\begin{minipage}{0.33\textwidth}
\includegraphics[width=\textwidth]{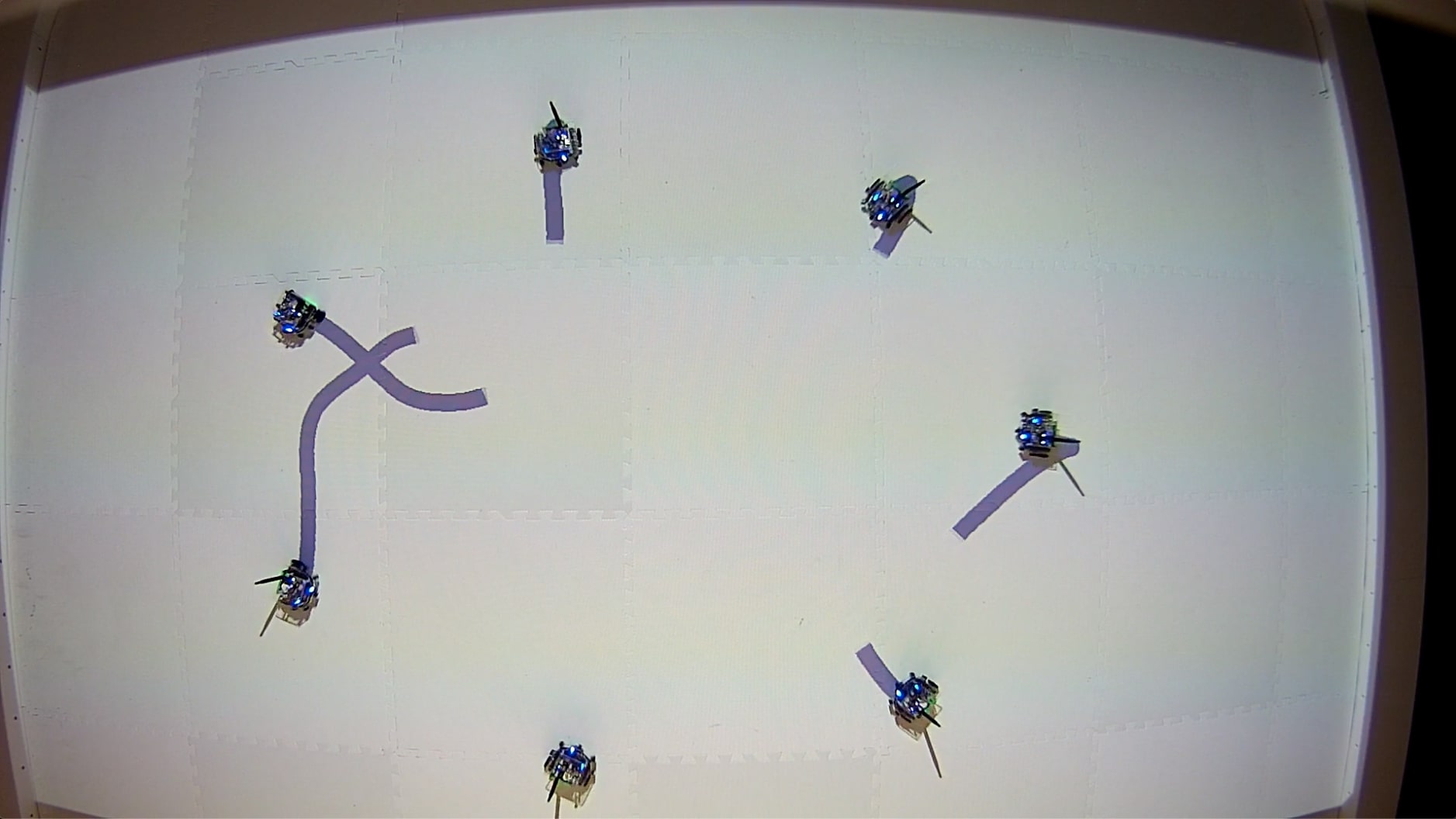}
\end{minipage}~%
\begin{minipage}{0.33\textwidth}
\includegraphics[width=\textwidth]{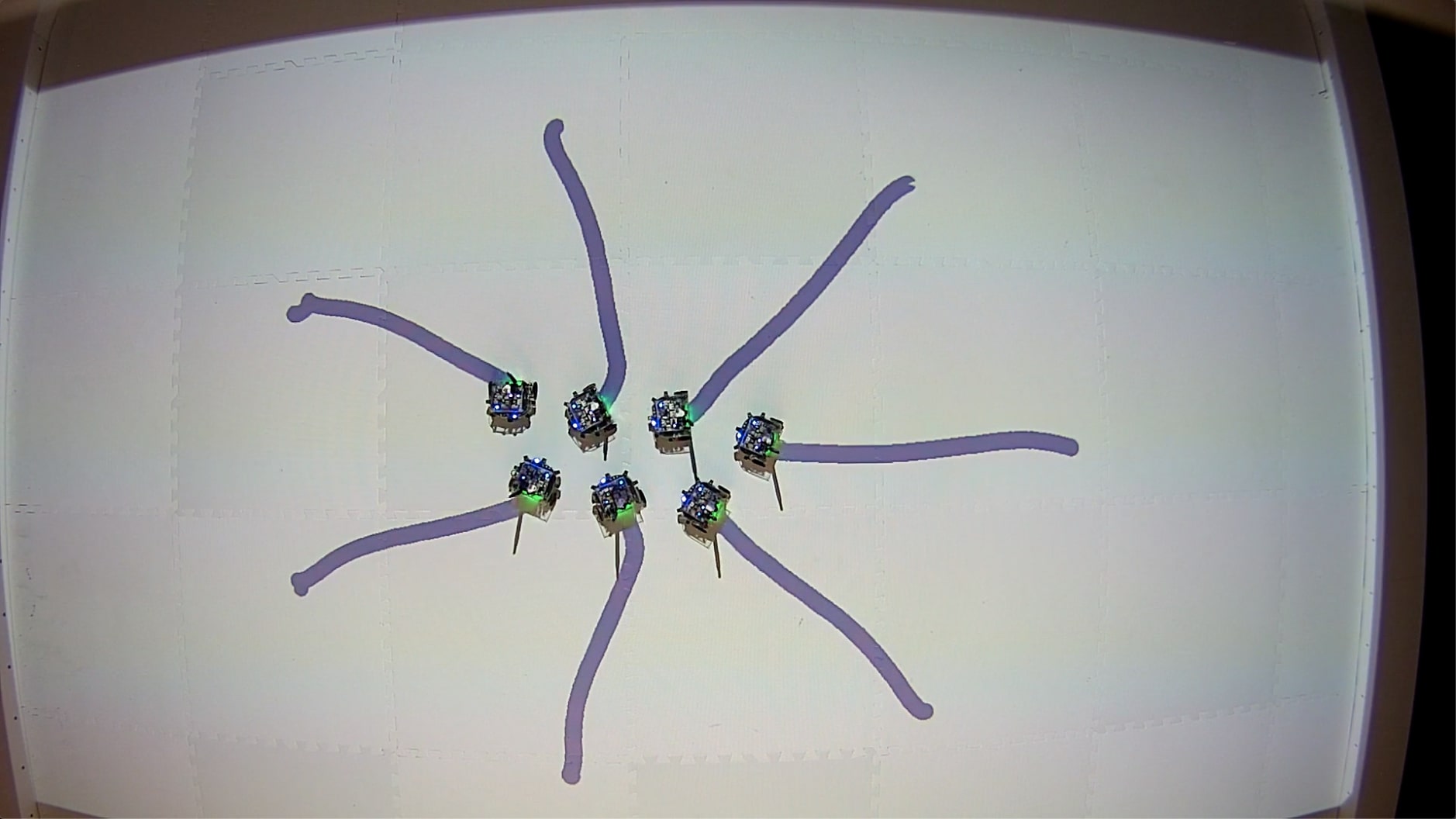}
\end{minipage}~%
\begin{minipage}{0.33\textwidth} 
\includegraphics[width=\textwidth]{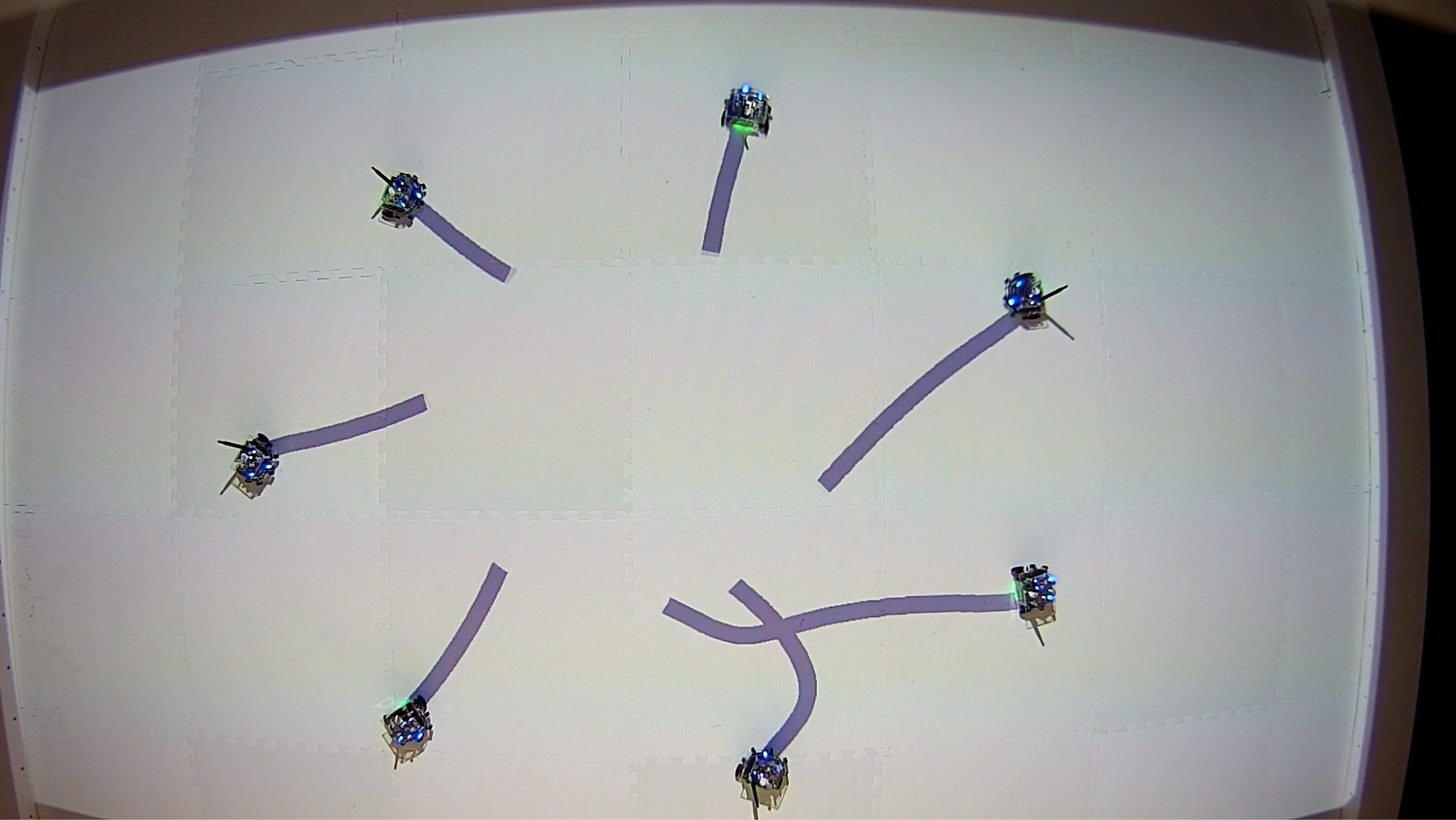}
\end{minipage}~%
\caption{Screenshots of experiment $2$. A group of $7$ GRITSBot-Xs completes an iteration of the repeated experiment detailed in sub-section~\ref{subsec:exp2}.  The robots are initially arranged on a circle (left) and attempt to traverse to the opposite side (right).  The Robotarium utilizes the multiplicative robust CBF based controller synthesis coupled with GPs for disturbance estimation to prevent collisions (middle) and ensure that each robot reaches the opposite side of the circle (right). A video including the experiment can be found at \texttt{https://youtu.be/DhzEktqnoCQ}.}
\label{fig:xbarrier}
\vspace{4mm}
\end{figure*}

\begin{figure*}[tb]
\centering
   \includegraphics[width=.49\textwidth]{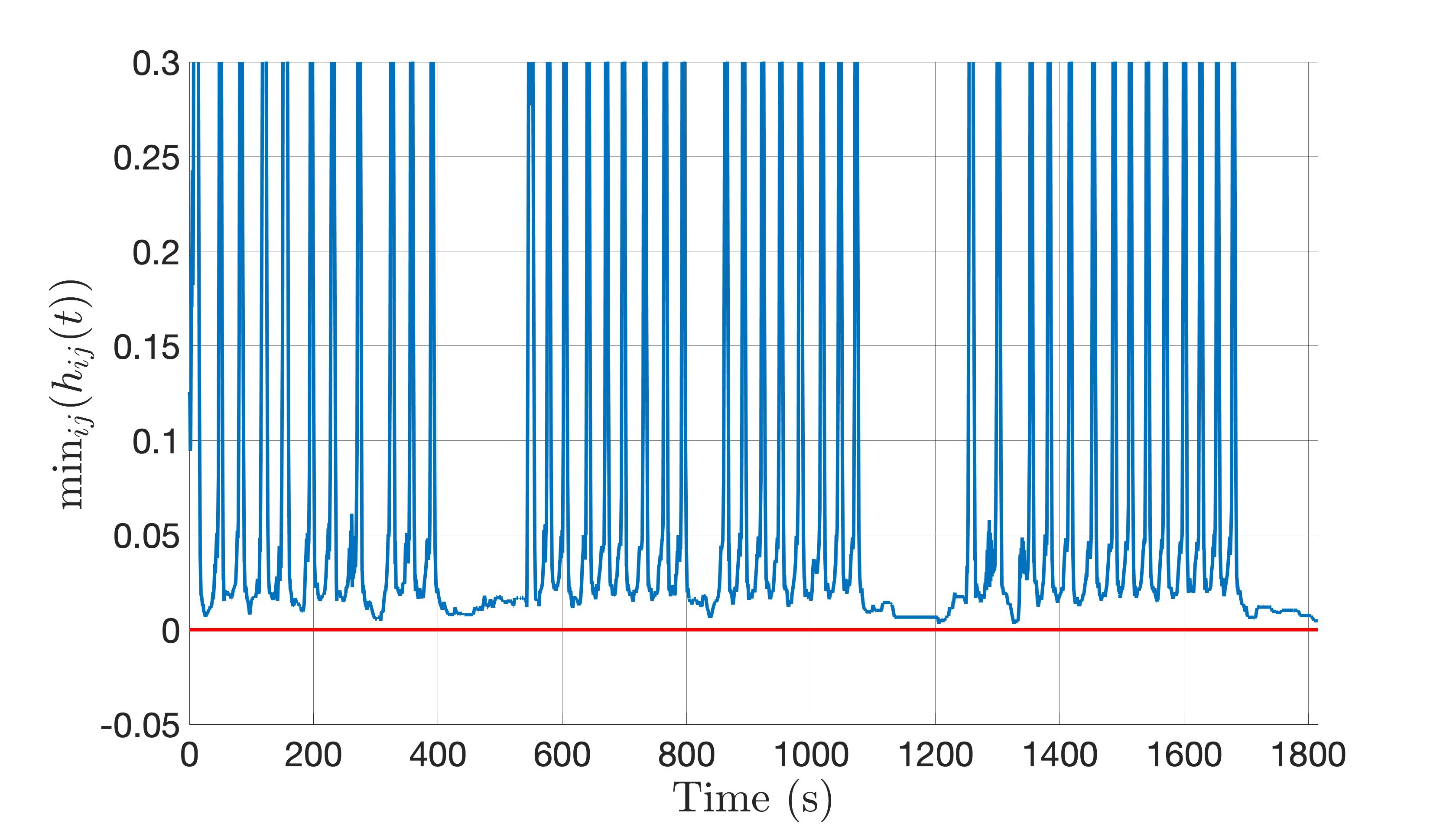}
   \includegraphics[width=.49\textwidth]{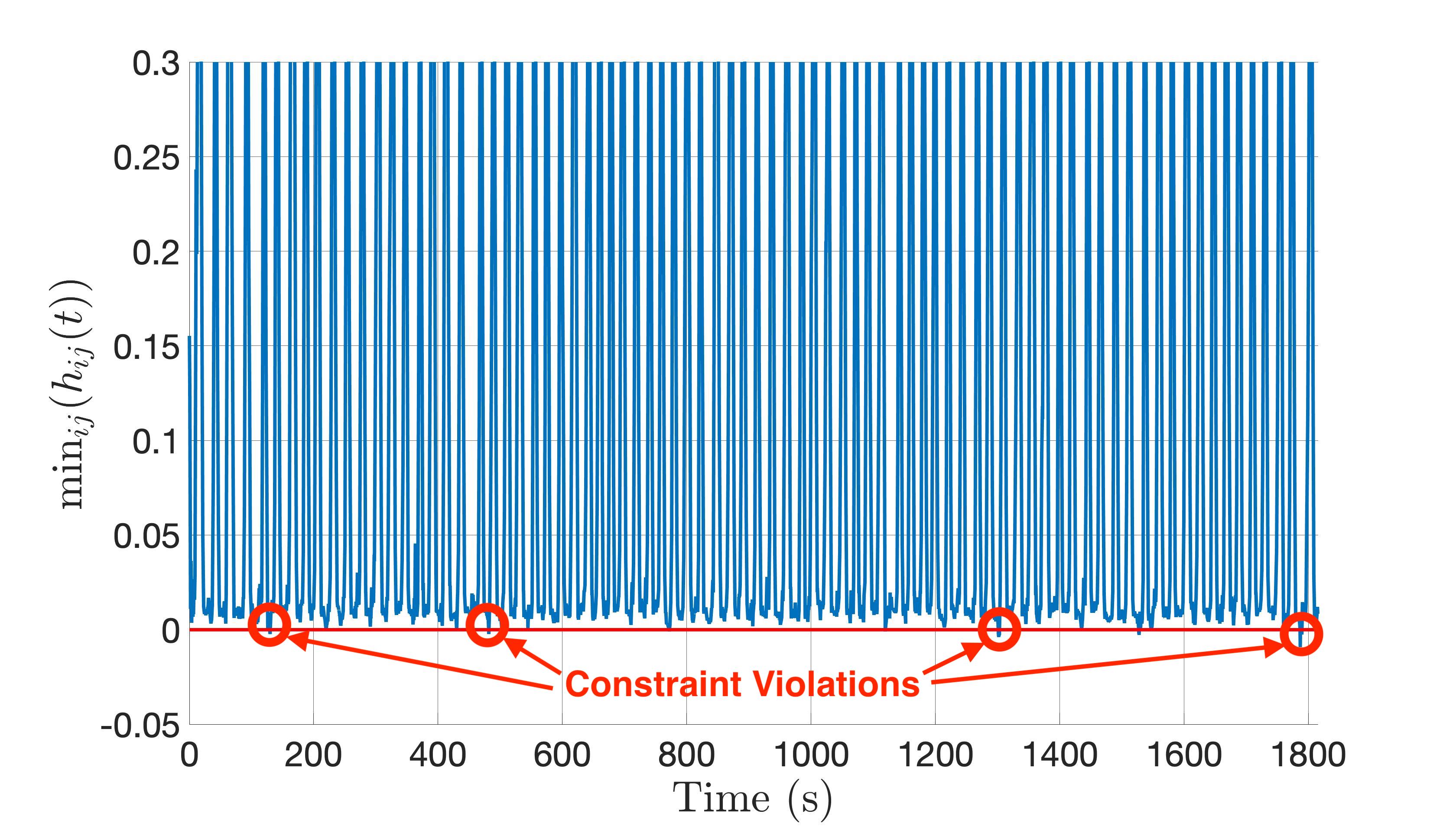}
\caption{Plots of $\min(h_{ij}(t))$ over time for the experiments with and without the robust CBF formulation (left and right, respectively).  At each point in time, if the value $\min(h_{ij}(t))$ is below the solid red line (i.e., the value is negative) then the collision-avoidance constraint is violated.  The robust CBF formulation (left) encounters zero constraint violations whereas the non-robust CBF formulation (right) does encounter constraint violations as shown in Table~\ref{tab:wct}.}
\label{fig:minh}
\end{figure*}

\subsection{Comparison with regular CBFs}
\label{subsec:exp2}
\todo{The objective of this experiment is to compare the robust and non-robust CBF formulations and to showcase how the robust formulation drastically decreases the number of constraint violations caused by network latency, model imperfection, and wheel slip issues relative to its non-robust counterpart.} The setup of the experiment is as follows. Initially, seven GRITSBot-Xs drive to a circular formation as shown in \autoref{fig:xbarrier}. Then, each robot is commanded to drive to the opposite end of the circle relative to their current position using a proportional controller. This maneuver is performed repeatedly over a $30$-minute time interval. Note that the maneuver causes the robots' paths to cross at the center of the circle as highlighted in \autoref{fig:xbarrier}. The experiment is performed for both the robust and non-robust formulations. In the case of the robust CBF, the disturbed portion of the dynamics was obtained offline via GP models as described in Section~\ref{sec:disturb-estimation}.

We compare both formulations by examining the frequency of collisions which occur\todo{, the number of maneuvers completed and the deviation between the nominal and applied control inputs} in each experiment as well as the wall-clock time of solving the QP. To compare \todo{the frequency of collisions}, for both experiments, we record the minimum $h_{ij}$ value at each time step and check if it is negative (i.e., a collision). The plot of the minimum value of $h_{ij}$ over time is shown in Figure~\ref{fig:minh}. As exhibited in Table~\ref{tab:wct}, $0$ constraint violations occur during the experiment using the robust CBF formulation. On the other hand, the non-robust formulation does periodically experience collisions, since the minimum value of $h_{ij}$ frequently drops below $0$ (with the first violation occurring after just $5$ repetitions). These violation mostly occur when the robots are clustered near the center of the circle and can be attributed to network latency, model imperfection, and wheel slip issues. 

\todo{Moreover, because the robust CBF formulation explicitly accounts for the disturbance, the framework generates a more conservative behavior than its non-robust counterpart, resulting in some performance degradation. This effect can be evaluated by inspecting the number of peaks in Figure~\ref{fig:minh}, each of which denotes the completion of a single maneuver. Specifically, the robots execute the maneuver $42$ times during the robust CBF experiment which is roughly half of the $82$ maneuvers executed during the non-robust CBF one. It is worth noting that one of the significant causes of this difference are three maneuvers in the robust CBF case which took significantly longer to complete relative to the others (shown by the three large gaps between peaks in the left plot of Figure~\ref{fig:minh}) most likely due to network latency. 

Another metric upon which we compare the performance of both methods is the squared norm of the difference between the nominal and applied control inputs over all time steps. As shown Table~\ref{tab:wct}, the mean squared norm difference in the robust CBF case is $19.08$, constituting roughly a $56\%$ increase over the $12.18$ mean of the non-robust CBF formulation. However, we note that this metric is less reliable than the number of maneuver completions since the application of each framework causes the robots to take different trajectories.}

Focusing on the run time, it is clear that the robust CBF formulation for multiplicative disturbances is fit for online operation since its mean solution frequency was $220$ Hz during the experiment as depicted in Table~\ref{tab:wct}. However, it is important to note that the robust CBF is slower than its non-robust counterpart, which can be attributed to the additional constraints added to the QP aimed at accounting for the disturbance ($2^m$). This addition of constraints occurs because in the case of multiplicative disturbances, checking for the worst-case corner point directly depends on the control input which is a decision variable in the QP. On the other hand, this issue does not arise in the case of additive disturbances where checking for the worst-case extreme point is independent of the control input as described in subsection~\ref{subsec:cbf_additive} and tested in \cite{emam2019robust}, where timing data and experimental verification is given for the additive formulation. The comparison of the wall-clock times of both experiments is shown in Table~\ref{tab:wct}. \todo{Overall, although the use of GPs for disturbance estimation only results in probabilistic guarantees, this experiment demonstrates how the robust-CBF formulation drastically reduces the number of constraint violations during long-term operation.}

\begin{table}[tb]
    \centering
        \caption{Comparison of the Wall-Clock Times (WCTs) for solving the Quadratic Programs with and without the robust CBF formulation. The time violation entry indicates the duration during which the constraint was violated for each experiment. \todo{The last two entries denote the mean and maximum squared norm difference between the nominal and applied control inputs over all time steps.}}
    \begin{tabular}{|c|c|c|}
        \hline
        & Robust CBF & Non-Robust CBF \\
        \hline
        \hline
       Avg. WCT (ms) & $4.580$ & $0.9411$ \\
       \hline
       Var. of WCTs ($\text{ms}^2$) & $12.200$ &  $0.020$\\
       \hline
       Avg. Freq. (Hz)  & $218$ & $ 1062$ \\
       \hline
       Time Violated (s) & $0$ & $11.682$ \\
       \hline
     \todo{ $\text{mean}(||u^*-u_{\text{nom}}||^2)$} &  \todo{$19.08$} &  \todo{$12.18$} \\
      \hline
        \todo{$\text{max}(||u^*-u_{\text{nom}}||^2)$} &  \todo{$73.83$} &  \todo{$100.81$} \\
       \hline 
    \end{tabular}
    \label{tab:wct}
\end{table}

\section{Conclusion}
\label{sec:conclusion}

This paper introduced a novel robust Control Barrier Function (CBF) framework that deals with disturbed dynamical systems by modelling the uncertainty as a union of convex hulls.  Using this set, we modelled the disturbed system as a differential inclusion which allowed us to leverage a rich set of mathematical tools, rendering our method applicable in many scenarios. Moreover, we discussed how the proposed robust CBFs can be coupled with data-driven methods aimed at disturbance estimation and introduced a controller-synthesis framework for disturbed control-affine systems. The results were then specialized to differential-drive robots and the efficacy of the approach was showcased in experiments on the Robotarium.

\bibliographystyle{unsrt}
\bibliography{citations.bib}






\end{document}